\documentclass[10pt,a4paper]{article}
\PassOptionsToPackage{table}{xcolor}
\usepackage{Risex-Lab-Paper}
\labsetup{
  name={RISE-X Lab},
  shortname={RISE-X Lab},
  tagline={School of Computer Science,\\Shanghai Jiao Tong University},
  report={},
  date={October 30, 2026},
  website={},
  paper-type={},
  masthead={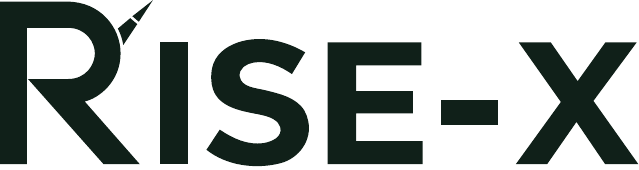},
  masthead-width={25mm},
  accent={123C2C},
  panel={F4F8F0},
  text-width={6.2in},
  title-size={16.5},
  title-leading={27},
  running-title-size={10},
  running-title-leading={0},
  date-size={10},
  date-leading={12}
}

\usepackage{amsmath,amsfonts,bm}

\def\eqref#1{equation~\ref{#1}}

\def\1{\bm{1}}

\DeclareMathAlphabet{\mathsfit}{\encodingdefault}{\sfdefault}{m}{sl}
\SetMathAlphabet{\mathsfit}{bold}{\encodingdefault}{\sfdefault}{bx}{n}

\usepackage{hyperref}
\hypersetup{hypertexnames=false}
\usepackage{url}
\usepackage{algorithm}
\usepackage{algorithmic}
\usepackage{booktabs}
\usepackage{multirow}
\usepackage{graphicx}
\usepackage[table]{xcolor}
\usepackage{pifont}
\usepackage{afterpage}
\usepackage{wrapfig}
\usepackage{amsthm}

\newtheorem{theorem}{Theorem}[section]
\newtheorem{lemma}[theorem]{Lemma}

\newcommand{\cmark}{\ding{51}}
\newcommand{\xmark}{\ding{55}}
\title{%
  \texorpdfstring{%
    \raisebox{-0.10em}{\includegraphics[height=1.2em]{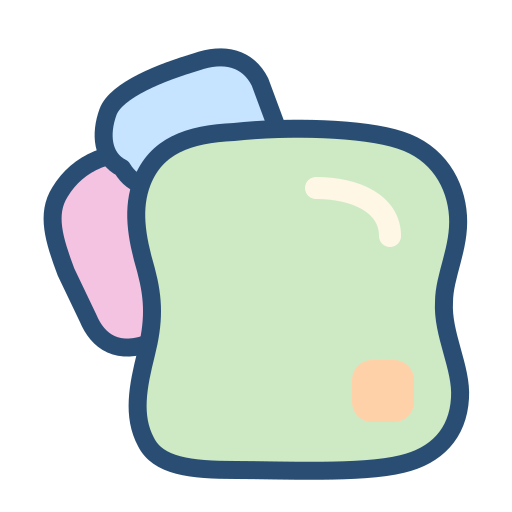}}%
    \hspace{0.20em}%
  }{}%
  In-Model Context Compression with Agentic Soft Tokens for Long-Horizon LLM Agents}

\author{\bfseries%
  \mbox{Xuning He\textsuperscript{1}}\quad
  \mbox{Kai Fan\textsuperscript{2}}\quad
  \mbox{Runyu Zhang\textsuperscript{1}}\quad
  \mbox{Binghuan Wang\textsuperscript{3}}\quad
  \mbox{Yuxiang Zhang\textsuperscript{4}}\\
  \mbox{Xue Jiang\textsuperscript{5}}\quad
  \mbox{Yihong Dong\textsuperscript{1,\Letter}}%
}
\affiliations{%
  \mbox{\textsuperscript{1}Shanghai Jiao Tong University}\quad
  \mbox{\textsuperscript{2}Nanyang Technological University}\quad
  \mbox{\textsuperscript{3}Nankai University}\quad
  \mbox{\textsuperscript{4}Tongji University}\quad
  \mbox{\textsuperscript{5}Peking University}%
}
\paperlinks{%
  \textbf{GitHub Repo:}\enspace
  \url{https://github.com/RISE-X-Lab/Agentic-Soft-Token}%
}
\correspondence{%
  \textbf{Correspondence:}\enspace
  \href{mailto:dongyh@sjtu.edu.cn}{dongyh@sjtu.edu.cn}%
}

\begin{document}

\makelabtitle{
Long-horizon LLM agents play an important role in solving complex and sustained tasks, with ever-growing interaction contexts. However, these accumulated contexts will raise substantial computational cost and obscure key decision-relevant evidence.
Existing methods typically train external compressors or context-management agents to rewrite context as shorter hard-token sequences, but such external selection may not align with the LLM's evolving representation and is necessarily lossy.
In this paper, we introduce \textbf{Agentic Soft-Token Compression (ASTC)}, a training-free compression method using soft tokens within LLM agent for long-horizon tasks.
Specifically, ASTC uses the model's native recurrent dynamics to determine the already absorbed information, leveraging the commonly-used hybrid Gated Delta Network (GDN)--attention.
It then carries this information forward as agentic soft tokens while retaining exact values as hard tokens.
The resulting compact sequence reduces later-layer computation and yields reusable KV caches, extending the savings across the agent's trajectory.
Extensive experiments on long-horizon tasks show that ASTC matches or exceeds the uncompressed agent's accuracy with up to 59.84\% fewer end-to-end Attention Pairs, while trained compressors require 12.3$\times$--194.9$\times$ more Attention Pairs per successful task.
These findings highlight agentic soft tokens as a model-native way to carry absorbed context forward as the task representation evolves.
}
\authornote{}{\Letter: Corresponding author.}

\section{Introduction}
\label{sec:introduction}

Large language models (LLMs) are increasingly deployed as agents that execute
multi-step workflows rather than produce isolated responses.
Each step appends system instructions, skills, earlier decisions, tool
interactions, code, and environmental feedback to the next model call. The
resulting context connects decisions across time, so losing an early constraint
or observation can trigger an inconsistent action whose effects propagate
through the remaining trajectory. Preserving the full context, however,
increases prefill and attention computation, enlarges the key--value (KV) cache,
and forces relevant evidence to compete with stale or redundant information.
Long-horizon agents therefore face the central challenge of controlling this
growing cost without disrupting the cross-step information flow required for
coherent decisions.

\afterpage{%
\begin{figure}[t!]
    \centering
    \includegraphics[width=\textwidth]{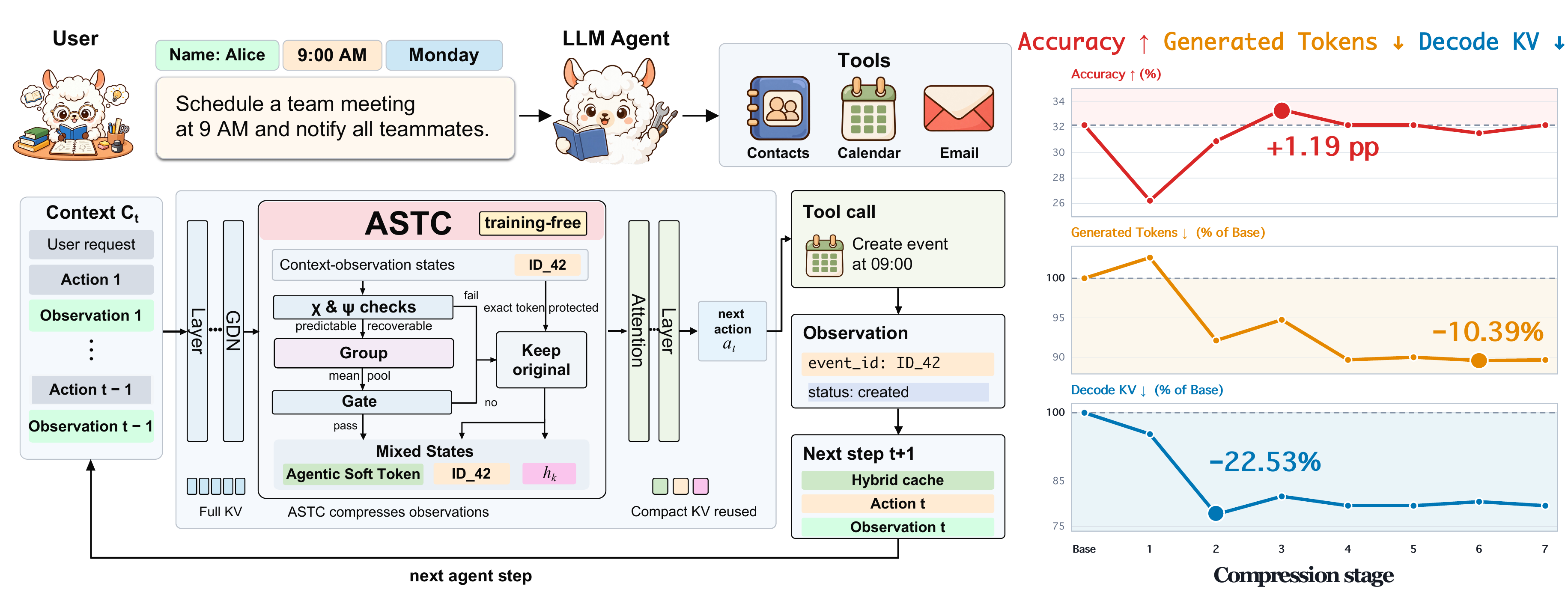}
    \caption{ASTC's agent framework and compression workflow (left), and its
    AppWorld performance across compression stages with Qwen3.5-9B (right).}
    \label{fig:overview}
\end{figure}
}

Reducing context under this constraint is not an ordinary summarization problem.
Agent context mixes semantic content with exact values.
Explanations, code, and repeated observations can become redundant once the
acting LLM has integrated their information, whereas identifiers, API arguments, and
numerical values may determine later actions and must remain exact. Information
that has little effect on the current action may also become decisive several
steps later. Such delayed relevance makes one-shot importance estimates
unreliable. Compressibility therefore cannot be inferred from length,
repetition, or semantic similarity alone. It depends on what the acting LLM has
already absorbed, what the current decision requires, and what a future step may
need to recover. Because compression recurs as the trajectory grows, each
operation must save more computation than it consumes. Effective context compression must
therefore compact information already absorbed by the acting LLM, preserve
exact or potentially useful evidence as the task evolves, and deliver savings
across subsequent decoding and agent steps.

Existing context-compression methods typically rely on training, using either
hard-token or soft-token compression. Hard-token methods prune or rewrite text,
risking irreversible loss of exact values~\citep{jiang2023llmlingua,
pan2024llmlingua2,xu2024recomp,kang2026acon,liu2025cat}. Soft-token methods
preserve continuous representations but require separately trained encoders or
compressors~\citep{mu2023gist,chevalier2023autocompressors,ge2024icae,
li2026endtoend}. More fundamentally, separately trained compressors create a
\emph{model-taste gap}. The information they learn to preserve may not match
what the acting LLM has already absorbed or will need later. This motivates an
in-model approach that uses the acting LLM's
recurrent state to determine compressibility and carries absorbed information
forward as agentic soft tokens.

In this paper, we introduce \textbf{Agentic Soft-Token Compression (ASTC)}, a
training-free approach that uses the acting LLM's own computation to decide
what its growing context must retain. The
early layers have read the full context and the later layers would
otherwise continue to process every original token.
Figure~\ref{fig:overview} places this boundary within a continuing agent
workflow and shows how the choice of compression stage changes task accuracy
and workload. In hybrid Gated Delta Network (GDN)--attention
models~\citep{yang2025gated}, ASTC uses the native recurrent state as a memory
of the context already processed. If a token's contribution is largely
predictable from the preceding state, ASTC further checks whether it remains
recoverable after the full observation has been processed. This guards against
treating information overwritten by later updates as redundant. For content
passing both checks, ASTC pools eligible hidden states into agentic soft tokens
for later layers, while retaining protected exact values, such as identifiers
and numbers, as explicit hard tokens. The shorter sequence reduces attention work for the current
prediction and creates smaller KV caches that remain reusable as the agent
continues. Each compression decision is therefore informed by what the model
has already absorbed, and its benefit extends across the trajectory.

Across diverse long-horizon tasks, ASTC maintains task quality close to
or above the uncompressed model while reducing end-to-end Attention Pairs by
up to 59.84\%. Preserving exact values avoids the sharp accuracy loss seen when
protection is removed, while the compact cache continues to save attention
work over later decisions with only small changes in action distributions.
The same training-free criterion extends beyond tool observations, matching
uncompressed accuracy with 19.48\% fewer Decode KV interactions under
full-history compression and reducing workload in additional long-context and
code-repair tasks. ASTC combines the highest
task accuracy with the lowest Attention Pairs per successful task, a
12.3$\times$--194.9$\times$ advantage over trained compressors.

Our contributions are summarized as follows:
\begingroup
    \setlength{\itemsep}{5pt}
    \setlength{\parsep}{0pt}
    \setlength{\parskip}{0pt}
    \setlength{\topsep}{10pt}
    \setlength{\partopsep}{0pt}
\begin{itemize}
    \item We identify the core challenge of long-horizon compression as reducing
    context cost without losing future-relevant evidence or diverging from the
    acting LLM's task representation.
    \item We introduce ASTC, a training-free method that uses native recurrent
    dynamics to retain exact values as hard tokens and carry absorbed content
    forward as soft tokens with reusable compact KV caches across subsequent
    agent steps.
    \item We bound trained-compressor objective mismatch, establish conditions
    under which ASTC has lower next-token distortion than equal-length
    hard-token representations, and quantify attention savings from cache reuse.
    \item Across long-horizon tasks, ASTC achieves comparable or better
    accuracy with up to 59.84\% fewer end-to-end Attention Pairs and a
    12.3$\times$--194.9$\times$ advantage in end-to-end Attention Pairs per
    successful task over trained compressors.
\end{itemize}
\endgroup

\begin{figure}[t!]
    \centering
    \includegraphics[width=\textwidth]{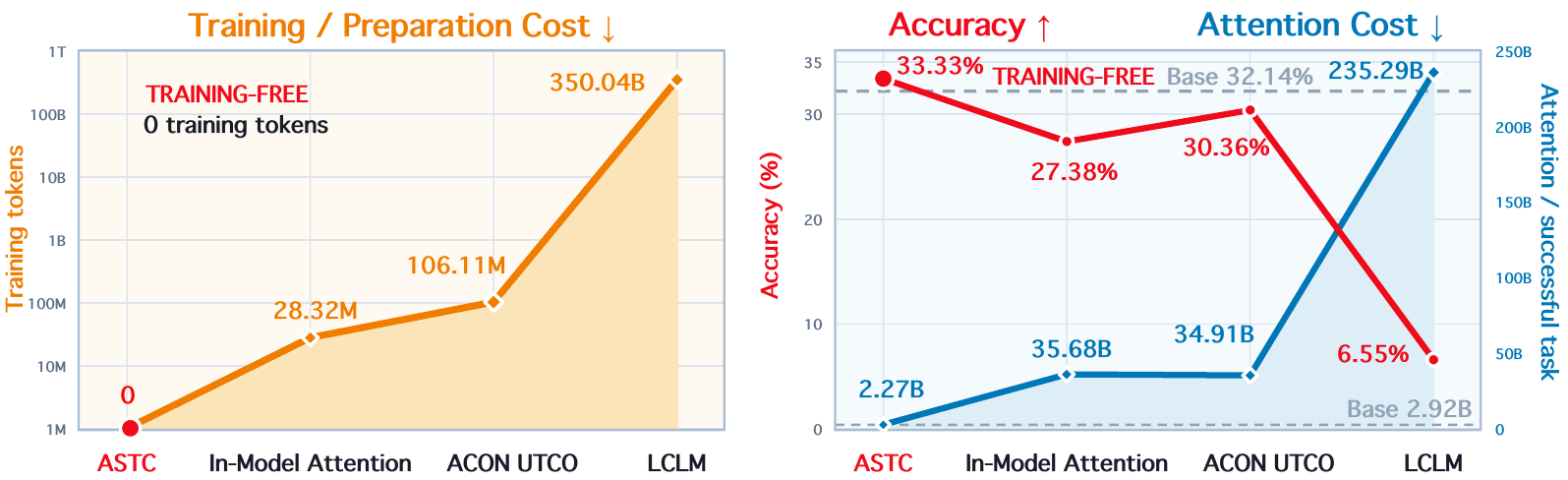}
    \caption{Training/preparation costs of ASTC, In-model OPD, ACON UTCO, and
    LCLM (left), and their AppWorld accuracy and Attention Pairs per successful
    task (right).}
    \label{fig:training-cost-results}
\end{figure}

\section{Motivation}
\label{sec:motivation}

\subsection{In-Model Soft Tokens}

Context compression determines both how many tokens remain and which evidence
reaches the acting LLM. Existing methods make this decision before the model
processes the compressed context, either pruning tokens or rewriting earlier
interactions as shorter text~\citep{jiang2023llmlingua,pan2024llmlingua2,
kang2026acon,liu2025cat}. This ordering creates a selection mismatch because the
value of a context span is not intrinsic to its text. It depends on how the
acting LLM represents the current task, what it has already absorbed, and which
action it is about to take; the same observation can be redundant at one step
and decisive at another. An external compressor must approximate this
model-dependent boundary from text alone. In a long-horizon trajectory, a
mistaken deletion can propagate because later decisions are made from the
retained context.Hard-token output adds a second bottleneck because retained information must be
serialized as a shorter sequence of vocabulary tokens, so exact structure or
features distributed across positions must be reconstructed as short text. Even
a correctly selected span can therefore lose task-relevant information when
forced into a short textual rewrite. Moving compression inside the model lets
the acting LLM judge compressibility where task information is represented, while
soft tokens retain information in continuous contextual representations.
In-model selection and soft-token representation are therefore complementary,
letting the acting LLM judge what matters and preserve that information in its own
representation space, thereby avoiding both external selection mismatch and
discrete textual reconstruction.

\newpage
\subsection{Tool Observations}

\begingroup
\setlength{\columnsep}{12pt}
\setlength{\intextsep}{3pt}
\begin{wrapfigure}{r}{0.50\textwidth}
\centering
\includegraphics[width=\linewidth]{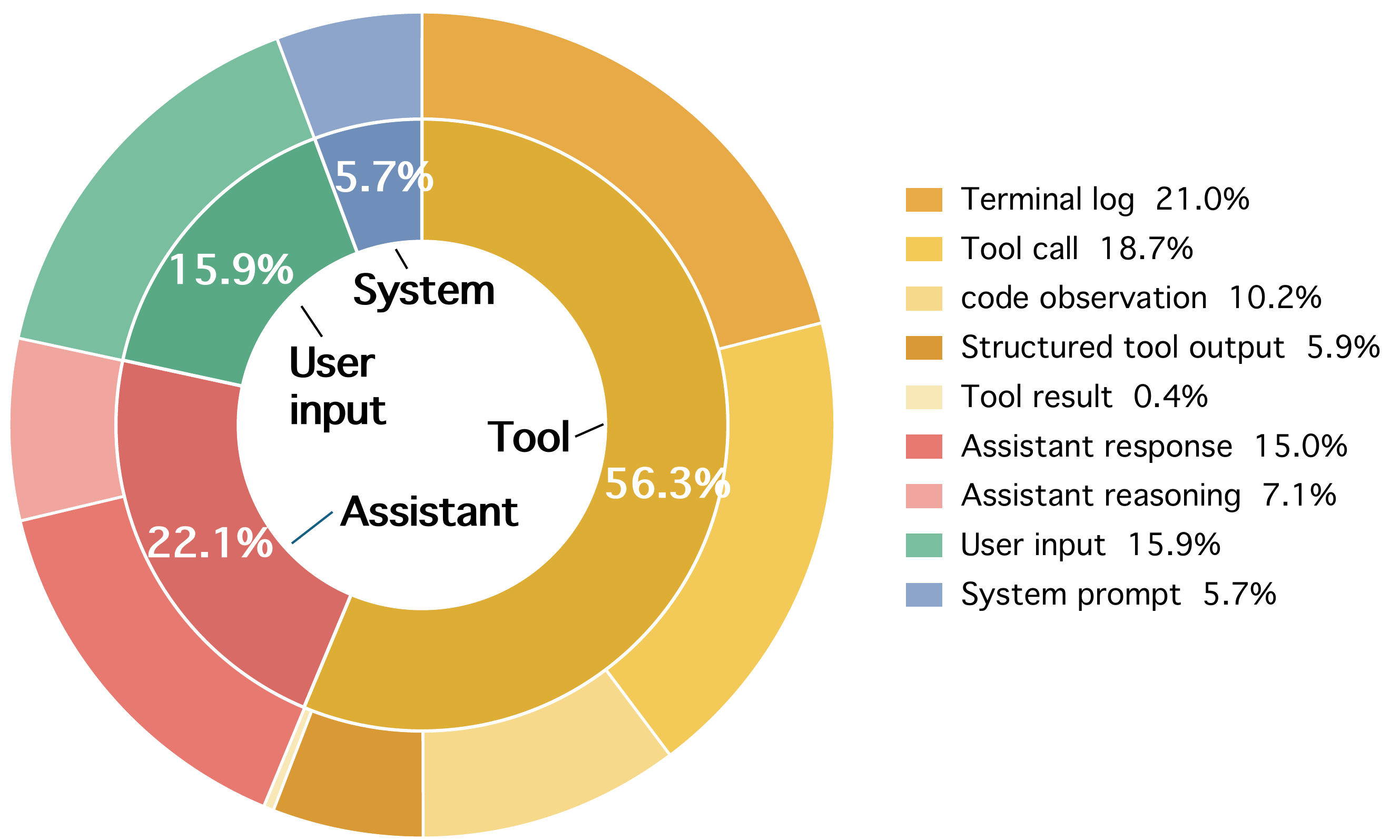}
\caption{Composition of accumulated context in real-world agent
trajectories.}
\label{fig:context-composition}
\end{wrapfigure}

We analyze agent trajectories from real-world user interactions across systems
such as Fable~5, GPT-5.6 Sol, Kimi~K3, and
GLM-5.2~\citep{anthropic2026fable,openai2026gpt56,moonshot2026kimik3,
zai2026glm52}. As Figure~\ref{fig:context-composition} shows, a consistent
pattern emerges across these trajectories. Tool observations occupy the
largest share of the accumulated context and continue to grow as the agent
interacts with its environment. In our experiments, tool observations are the
primary compression target.

Tool observations are difficult to compress because they mix compressible
natural-language and code spans with exact values, such as paths, identifiers,
and checksums, that later actions may consume verbatim. Our preliminary
experiments show that the former tolerates aggressive compression, whereas
changing a single character in the latter can invalidate a later action. This
heterogeneity makes tool observations a demanding but informative testbed.
Effective agent compression must therefore compact semantic information
distributed across many tokens without altering the exact values on which
future actions depend.
\par
\endgroup

\subsection{Model-Native Selection}

Our initial approach followed existing work by training a separate in-model
compressor to produce soft tokens. We inserted an attention--MLP compressor
after either a single backbone layer or multiple layers, thereby compressing
the hidden sequence once or repeatedly, and trained it with
\textbf{On-policy distillation (OPD)}. Comparing these variants with ACON and
LCLM reveals a consistent pattern: training-based compression incurs
substantial preparation cost yet delivers weak task performance, while
repeated intervention across more layers offers no remedy
(Figure~\ref{fig:training-cost-results})~\citep{kang2026acon,li2026endtoend}.
This failure reflects a difficult optimization problem. An auxiliary
compressor must infer what the acting LLM has already absorbed, retain
information that later steps may require, and emit features compatible with
downstream layers, all through an objective separate from the LLM's native
prediction process.

Hybrid GDN--attention models already compute a signal closely aligned with this
decision during their normal forward pass. A GDN recurrent state acts as an
online associative memory. Its gated delta update can be written
as~\citep{yang2025gated}
\begin{equation}
    \mathbf{S}_t
    = \alpha_t \mathbf{S}_{t-1}
    + \beta_t \mathbf{r}_t \mathbf{k}_t^{\top},
    \qquad
    \mathbf{r}_t
    = \mathbf{v}_t - \alpha_t \mathbf{S}_{t-1}\mathbf{k}_t,
    \label{eq:gdn-update}
\end{equation}
where $\alpha_t$ controls how much previous memory survives, $\beta_t$ controls
the write strength, and $\mathbf{r}_t$ is the part of the incoming value that
cannot be recovered from the retained state. Information already represented
by the state produces a small correction, whereas novel or conflicting
information changes the memory more strongly. The update therefore measures
information relative to what the acting LLM currently remembers, not importance
inferred from surface form alone. Moreover, the gate and delta update are
optimized jointly with the LLM's predictions, so this distinction shares the
representation space and task state used by downstream layers.
Reusing GDN dynamics avoids teaching an auxiliary compressor to approximate the acting
LLM's own notion of novelty. Training-free selection is therefore not merely a
cheaper alternative; it follows from recognizing that the acting LLM has
already performed the information judgment that compression requires.

\section{Related Work}
\label{sec:related-work}

Discrete context compression shortens the sequence before it reaches the acting
LLM. Selective Context~\citep{li2023selective} prunes redundant spans;
LLMLingua~\citep{jiang2023llmlingua} and
LongLLMLingua~\citep{jiang2024longllmlingua} rank tokens for prompt pruning;
LLMLingua-2~\citep{pan2024llmlingua2} learns selection from distilled labels;
and RECOMP~\citep{xu2024recomp} trains compressors for retrieved evidence.
For extended agent workflows, ACON~\citep{kang2026acon} optimizes textual
compression guidelines from task failures, while Context as a Tool
(CAT)~\citep{liu2025cat} trains SWE-Compressor to invoke context management
during execution.
MEM1~\citep{zhou2025mem1} instead trains a compact memory jointly with agent
reasoning. While these lines reduce accumulated context, text-level selection
can discard exact evidence or diverge from the acting LLM's evolving task
representation, and learned memory policies require additional training.

Soft-token methods retain information in continuous states. Gist
Tokens~\citep{mu2023gist} learns compressed prompt activations,
AutoCompressors~\citep{chevalier2023autocompressors} produce recursive summary
vectors, and ICAE~\citep{ge2024icae} trains compact memory slots. Sentinel
Tokens~\citep{ren2023sentinel} and Activation
Beacon~\citep{zhang2024beacon} compress intermediate activations for later
processing. DAST~\citep{chen2025dast} allocates soft tokens according to
contextual information density, while Latent Context Language Models
(LCLMs)~\citep{li2026endtoend} train dedicated encoder--decoder models for
large-scale compression. For tool-using LMs, \citet{xu2024toolcompress}
combine compressed representations with raw
tokens for exact tool interfaces. These studies show that continuous memory can
support incremental reuse and coexist with exact-token retention. Their
compression behavior is nevertheless learned through model adaptation or
auxiliary components, and their primary settings differ from a growing history
of agent actions and observations.

ASTC brings the continuous memory studied in soft-token compression to the
growing histories addressed by agent context-management work. By operating
within the acting LLM without a separately trained compressor, it keeps
compression tied to the model's evolving task representation. To the best of our knowledge, ASTC is among
the first \emph{training-free, in-model} soft-token methods for context
compression in long-horizon LLM agents.

\begin{figure}[t!]
    \centering
    \includegraphics[width=\textwidth]{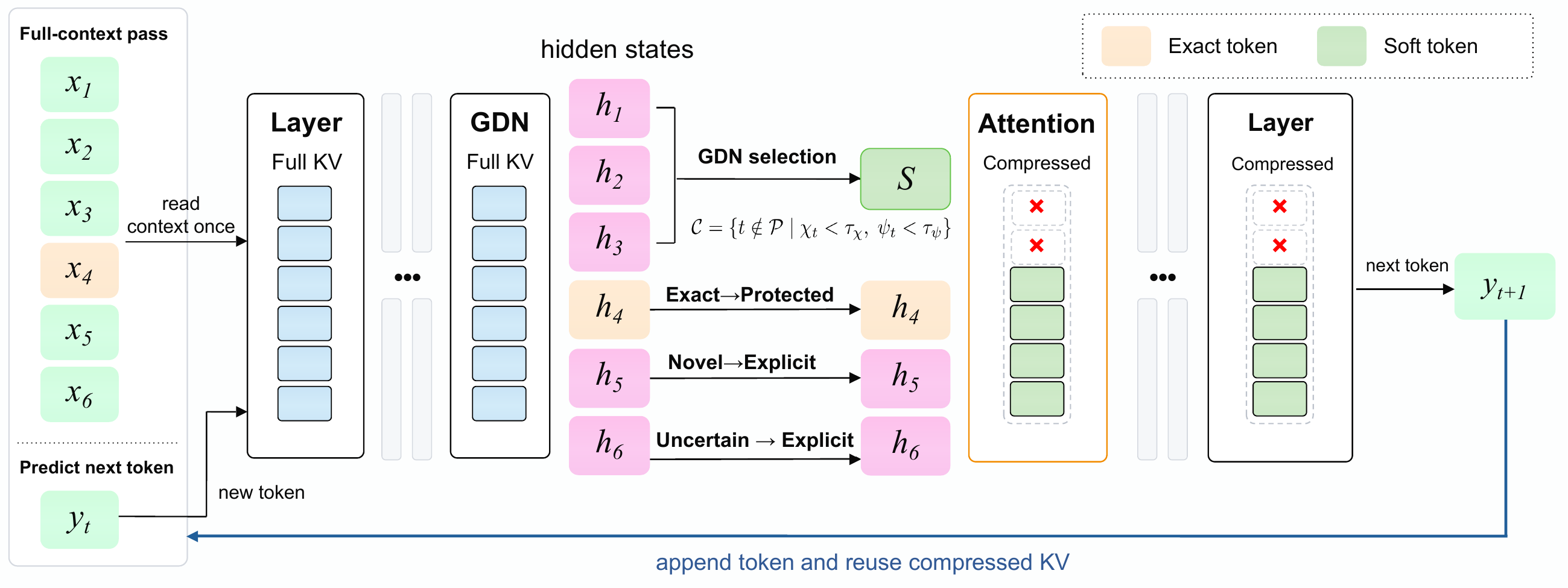}
    \caption{Overall architecture and workflow of ASTC, including in-model
    compression and KV cache reuse.}
    \label{fig:in-model-compression}
\end{figure}

\section{Agentic Soft Token}
\label{sec:method}

\subsection{Agentic Soft-Token Compression}

At step $a$, early layers process every token in the uncached suffix
$\Delta X^{(a)}$, using cache $\mathcal K_<^{(a-1)}$ for prior context.
ASTC then compresses the new hidden states in tool-observation spans
$\mathcal O$ subject to protection mask $P$:
\begin{equation}
    \begin{aligned}
    H_{\ell^\star}^{\mathrm{new}}
    &=F_{0:\ell^\star}(\Delta X^{(a)};\mathcal K_{<}^{(a-1)})
    \in\mathbb{R}^{m\times d},\\
    \widetilde H_{\ell^\star}^{\mathrm{new}}
    &=\mathcal{C}(H_{\ell^\star}^{\mathrm{new}};\mathcal{O},P)
    \in\mathbb{R}^{\widetilde m\times d},\quad \widetilde m\leq m.
    \end{aligned}
    \label{eq:astc-boundary}
\end{equation}

$P$ keeps recognized exact values and the action boundary as hard tokens,
although it may miss some exact fields. For each GDN head $h$, let
$\mathbf S_{i-1,h}$ be its state before position $i$, with key
$\mathbf k_{i,h}$, value $\mathbf v_{i,h}$, and decay
$\alpha_{i,h}=\exp(g_{i,h})$. Let $\mathbf v_i$ and
$\boldsymbol\delta_i$ concatenate the head-wise vectors. The first test
measures innovation relative to the recurrent state:
\begin{equation}
    \boldsymbol{\delta}_{i,h}
    = \mathbf{v}_{i,h}-\alpha_{i,h}\,
      \mathcal{R}(\mathbf{S}_{i-1,h},\mathbf{k}_{i,h}),
    \qquad
    \chi_i =
    \frac{\operatorname{RMS}(\boldsymbol{\delta}_i)}
         {\operatorname{RMS}(\mathbf{v}_i)+\epsilon},
    \label{eq:astc-innovation}
\end{equation}
where $\mathcal R(\mathbf S,\mathbf k)=\mathbf S\mathbf k$. A small
$\chi_i$ means the incoming value was largely predicted by the preceding
state. Prediction at position $i$ does not ensure that later updates retain
the value, so ASTC also reads the observation-end state $\mathbf S_{o,h}$
with a fixed, untrained spectral basis
$\mathbf Q_h\in\mathbb R^{d_{k,h}\times M}$, using the pooling slot budget $M$:
\begin{equation}
    \widehat{\mathbf{v}}_{i,h} =
    \mathbf{S}_{o,h}\mathbf{Q}_h
    (\mathbf{Q}_h^\top\mathbf{Q}_h+\lambda\mathbf{I})^{-1}
    \mathbf{Q}_h^\top\mathbf{k}_{i,h},
    \qquad
    \psi_i =
    \frac{\operatorname{RMS}(\mathbf{v}_i-\widehat{\mathbf{v}}_i)}
         {\operatorname{RMS}(\mathbf{v}_i)+\epsilon}.
    \label{eq:astc-readability}
\end{equation}
Here $\widehat{\mathbf v}_i$ concatenates the per-head reads. ASTC selects
an unprotected position only when it is both locally predictable
($\chi_i<\tau_\chi$) and still readable at the endpoint
($\psi_i<\tau_\psi$).

The selected positions still need a compact representation. ASTC partitions
them into at most $M$ causal groups, trims outliers around a member of each
group, and mean-pools the retained states of group $G$ at the boundary:
\begin{equation}
\mathbf z_G=\operatorname{cast}\!\left(\frac{1}{|G|}
\sum_{i\in G}\mathbf h_i^{\mathrm{FP32}}\right).
\label{eq:native-soft-token}
\end{equation}
The replacement stays in the acting model's hidden space rather than being
rewritten as text.
Only multi-position groups with maximum post-cast relative hidden-state
error at most $\tau_h$ become \emph{agentic soft tokens}. Protected, rejected, and
unmerged positions remain hard tokens. Each soft token takes its group's
latest logical position, and a stable sort restores causal order.

Compression leaves the pre-boundary caches at their original-token lengths,
while downstream attention caches the shorter mixed sequence. This cache
is reused during decoding and later agent steps, so an accepted merge saves
work beyond the prefill in which it occurs. Algorithm~\ref{alg:astc}
specifies the per-observation commit gate; bounded pooling is detailed in
Appendix~\ref{app:experimental-setting}.

\begin{algorithm}[H]
\caption{ASTC at one prefill step with a persistent hybrid cache}
\label{alg:astc}
\begin{algorithmic}[1]
\small
\REQUIRE Uncached suffix $\Delta X^{(a)}$, existing layer-specific hybrid
cache $\mathcal K^{(a-1)}$, closed observation spans $\mathcal O$ in this
suffix, protection mask $P$, frozen $F$, boundary $\ell^\star$, reader
$\mathbf Q$, thresholds $\tau_\chi,\tau_\psi,\tau_h$, group cap $M$,
budget $B$, remaining-layer count
$L_{\mathrm{rem}}$, fraction $\rho$
\ENSURE Next-token logits; updated hybrid cache $\mathcal K^{(a)}$
\FOR{each observation $o\in\mathcal O$}
    \STATE $\mathcal E_o\leftarrow\{i\in o:P_i=0\ \text{and $i$ is an original position}\}$
    \STATE $r_o\leftarrow\max(\lceil B/L_{\mathrm{rem}}\rceil,
    \lceil\rho|\mathcal E_o|\rceil)$
    \STATE Bypass $o$ if it already contains a soft token or
    $|\mathcal E_o|<r_o+1$
\ENDFOR
\STATE Run $F_{0:\ell^\star}$ on $\Delta X^{(a)}$ with its existing cache;
update early-layer caches and obtain $H_{\ell^\star}^{\mathrm{new}}$
\STATE During this run, capture native GDN tensors and read endpoint states
only for non-bypassed spans
\FOR{each non-bypassed observation $o\in\mathcal O$}
    \STATE $\mathcal C_o\leftarrow
    \{i\in\mathcal E_o:\chi_i<\tau_\chi\}$
    \STATE Continue to the next observation if $|\mathcal C_o|<r_o+1$
    \STATE $\mathcal C_o\leftarrow
    \{i\in\mathcal C_o:\psi_i<\tau_\psi\}$
    \STATE Continue to the next observation if $|\mathcal C_o|<r_o+1$
    \STATE $(\mathcal M_o,\mathcal Z_o)\leftarrow
    \textsc{BoundedPool}(H_{\ell^\star}^{\mathrm{new}}[\mathcal C_o],M,\tau_h)$
    \IF{$|\mathcal M_o|-|\mathcal Z_o|\geq r_o$}
        \STATE Replace $\mathcal M_o$ by $\mathcal Z_o$ at each group's greatest
        logical position
    \ENDIF
\ENDFOR
\STATE Stably sort the mixed new suffix by logical position
\STATE Run $F_{\ell^\star+1:L}$ on that suffix with its existing cache;
update downstream attention KV and GDN recurrent and convolutional states
\STATE Apply final norm and LM head to the last protected original position
\RETURN next-token logits and $\mathcal K^{(a)}$
\end{algorithmic}
\end{algorithm}

Here $B$ is the minimum downstream token--layer reduction,
$L_{\mathrm{rem}}$ counts layers after the boundary, and $\rho$ is the
minimum reduction as a fraction of eligible positions in each observation.

\section{Theoretical Analysis}
\label{sec:theory}

We examine what native selection can certify, how a compact cache saves work
over time, and how merging can shift downstream attention. The distributional
claims below concern a fixed next-token prefix; they do not imply identical
agent trajectories.

\subsection{Model-Native Selection and Continuous Representation}

Fix a history, a next-token prefix, and a boundary $\ell^\star$. Let $H$ be
the full boundary sequence and $Z$ a length-$k$ replacement at fixed logical
positions. We measure the acting model's next-token distortion by
\begin{equation}
D_\theta(H,Z)=D_{\mathrm{KL}}\!\left(
p_\theta^{>\ell^\star}(\cdot\mid H)\,\|\,
p_\theta^{>\ell^\star}(\cdot\mid Z)\right).
\label{eq:model-taste}
\end{equation}
A separately trained compressor minimizes a surrogate objective
$J_{\mathrm{tr}}$, which need not match $D_\theta$. On the same nonempty
admissible set $\mathcal A_k$, if both minima are attained,
$\sup_{Z\in\mathcal A_k}|D_\theta(H,Z)-J_{\mathrm{tr}}(H,Z)|
\leq\varepsilon$, and the returned $Z_{\mathrm{tr}}$ is within $\xi$ of
the surrogate minimum, then
\begin{equation}
0\leq D_\theta(H,Z_{\mathrm{tr}})
-\min_{Z\in\mathcal A_k}D_\theta(H,Z)
\leq 2\varepsilon+\xi.
\label{eq:taste-gap}
\end{equation}
This decomposition exposes a possible source of error, but does not establish
a nonzero gap or rank methods with different feasible sets.

ASTC instead reads the acting model's own GDN signals. The update in
Equation~\ref{eq:gdn-update} yields a per-head write bound proportional to $\chi_i$,
while $\psi_i$ measures what remains readable at the observation endpoint.
These tests describe state coverage, but replacing positions can also shift
downstream attention and recurrent states. For a fixed next-token prefix,
suppose those effects admit aligned local perturbations $e_j$ and the
remaining logit map is $L_j$-Lipschitz at each site. Write $\mathbf l$ and
$\widetilde{\mathbf l}$ for their logits, and $p$ and $\widetilde p$ for
their next-token distributions. With $E=\sum_jL_je_j$, we obtain
\begin{equation}
\|\widetilde{\mathbf l}-\mathbf l\|_\infty\leq E,
\qquad D_{\mathrm{KL}}(p\|\widetilde p)\leq 2E.
\label{eq:fixed-prefix-stability}
\end{equation}
Hence the next-token distribution changes little when $E$ is small, and the
greedy token is unchanged if its original logit margin exceeds $2E$.
The selector tests alone do not bound every term in $E$.

Stability addresses whether a replacement is safe; the representation space
determines what it can retain at a fixed budget. At equal physical length
$k$ and fixed logical positions in an unchanged context, hard-token
reachable boundary states form
a finite set $\mathcal A_k^{\mathrm{hard}}$ inside the continuous space
$\mathcal A_k^{\mathrm{soft}}=\mathbb R^{k\times d}$. With full-context
distribution $p_H=p_\theta^{>\ell^\star}(\cdot\mid H)$, let
$\delta_k(H)=\min_{Z\in\mathcal A_k^{\mathrm{hard}}}
\|p_H-p_\theta^{>\ell^\star}(\cdot\mid Z)\|_1$.
Set inclusion and Pinsker's inequality
give
\begin{equation}
\begin{aligned}
\inf_{Z\in\mathcal A_k^{\mathrm{soft}}}D_\theta(H,Z)
&\leq\min_{Z\in\mathcal A_k^{\mathrm{hard}}}D_\theta(H,Z),\\
D_\theta(H,Z_{\mathrm{ASTC}})
&<\min_{Z\in\mathcal A_k^{\mathrm{hard}}}D_\theta(H,Z)
\quad\text{if }E<\delta_k(H)^2/4.
\end{aligned}
\label{eq:soft-space}
\end{equation}
The first inequality describes representational capacity. The second gives
a sufficient condition for ASTC's particular soft-token output to outperform
every length-$k$ hard-token representation. Its premise is not guaranteed by
the selector thresholds; proofs appear in Appendix~\ref{app:theory-derivations}.

\subsection{One-Pass and Amortized Efficiency}

The boundary also determines when savings arise. Early layers still read the
new block, but later attention sees fewer positions in the same pass; the
avoided work must exceed selection and pooling cost for a net gain. Once the
shorter KV cache persists, subsequent queries revisit fewer keys.
\begin{theorem}[Amortized attention-pair reduction]
\label{thm:amortized-attention}
Suppose the affected downstream layers contain $A_{>}$ full causal attention
layers. Before one commit, the current prefill has physical length
$n_{\mathrm{phys}}$, including an unchanged cached prefix, and all $r$
removed positions belong to the newly processed block. Under the same
continuation of $N_q$ later queries reaching these layers, assume the
compared caches retain their $r$-position difference. Then the total
reduction in query--key token pairs is
\begin{equation}
\Delta Q=A_{>}r\left(
\frac{2n_{\mathrm{phys}}-r+1}{2}+N_q\right).
\label{eq:amortized-attention}
\end{equation}
\end{theorem}
The first term counts the current prefill, while $N_q$ counts reuse of the
same compact cache. Savings therefore accumulate without another compression
decision. The proof in Appendix~\ref{app:amortized-proof} concerns attention
pairs rather than wall-clock speedup.

\subsection{Downstream Attention Reweighting}

Compression can also change which evidence downstream attention emphasizes.
A redundant group may attract attention through multiplicity alone. The
following idealized lemma isolates that effect.

\begin{lemma}[Attention mass under idealized merging]
\label{lem:attention-refinement}
For the same downstream query in both runs, after the last logical position
of a merged group $G$ ($|G|=g>1$), assume its members and representative
share score $s$ and value $\mathbf v_G$. Let retained positions have
unchanged scores and values, with unnormalized mass $Z_R>0$ and
attention-weighted value
$\overline{\mathbf v}_R$. Their total attention masses before and after
replacement, $A_R$ and $\widetilde A_R$, satisfy
\begin{equation}
\widetilde A_R-A_R
=\frac{Z_R(g-1)e^s}
{(Z_R+e^s)(Z_R+ge^s)}>0,
\qquad
\mathbf u^\top(\widetilde{\mathbf o}-\mathbf o)
=(\widetilde A_R-A_R)
\mathbf u^\top(\overline{\mathbf v}_R-\mathbf v_G)
\label{eq:attention-refinement}
\end{equation}
for any linear readout direction $\mathbf u$.
\end{lemma}
The group retains one representative but loses the extra attention weight
of its $g-1$ copies. If the retained value has a larger projection along
$\mathbf u$ than $\mathbf v_G$, the attention output shifts toward it.
This describes an idealized local effect, not a guarantee of improved task
accuracy. Appendix~\ref{app:attention-refinement-proof} proves the lemma
and bounds the unequal-score and unequal-value cases separately.

\section{Experiments}
\label{sec:experiments}

We evaluate ASTC with a frozen Qwen3.5-9B acting model on
AppWorld~\citep{trivedi2024appworld},
OfficeBench~\citep{wang2024officebench}, and
8-Objective QA~\citep{zhou2025mem1,kang2026acon}.
We compare with uncompressed Base, prompt-based deletion (Prompting),
LLMLingua~\citep{jiang2023llmlingua},
LLMLingua-2~\citep{pan2024llmlingua2},
ACON UT and UTCO~\citep{kang2026acon}, and
LCLM~\citep{li2026endtoend}.
We report two configurations that differ in compression scope:
\textbf{ASTC (Tools)} compresses eligible tool observations, while
\textbf{ASTC (Context)} extends compression to eligible spans across the full
context. Both configurations use the same training-free compression mechanism
with identical hyperparameter settings and preserve protected exact values as
hard tokens. Unless otherwise stated,
subsequent analyses use ASTC (Tools).
We measure task quality and end-to-end Decode KV, Attention Pairs, and Generated
Tokens. Decode KV counts decoding query--cache interactions; Attention Pairs
counts causal query--key pairs across prefill, decoding, and auxiliary calls.
Detailed setups and additional experiments are in
Appendices~\ref{app:cache-comparators}--\ref{app:experimental-details}.

\subsection{Main Results}

\begin{table}[!t]
\caption{Accuracy and end-to-end inference workload on three long-horizon agentic
benchmarks.}
\label{tab:main-results}
\centering
\setlength{\tabcolsep}{1.5pt}
\renewcommand{\arraystretch}{1.08}
\resizebox{\textwidth}{!}{%
\begin{tabular}{@{}lc*{13}{c}@{}}
\toprule
& & \multicolumn{4}{c}{\textbf{AppWorld}}
& \multicolumn{4}{c}{\textbf{OfficeBench}}
& \multicolumn{5}{c}{\textbf{8-Objective QA}} \\
\cmidrule(lr){3-6}\cmidrule(lr){7-10}\cmidrule(l){11-15}
Method & Training-free
& Acc.\% $\uparrow$ & D-KV\% $\downarrow$ & Attn.\% $\downarrow$ & Gen.\% $\downarrow$
& Acc.\% $\uparrow$ & D-KV\% $\downarrow$ & Attn.\% $\downarrow$ & Gen.\% $\downarrow$
& EM\% $\uparrow$ & F1\% $\uparrow$ & D-KV\% $\downarrow$ & Attn.\% $\downarrow$ & Gen.\% $\downarrow$ \\
\midrule
Base & --
& \underline{32.14} & 100.00 & 100.00 & 100.00
& 48.42 & 100.00 & 100.00 & 100.00
& 28.91 & 37.51 & 100.00 & 100.00 & 100.00 \\
\midrule
\multicolumn{15}{@{}l}{\textit{Hard-token compression}} \\
Prompting & \cmark
& 25.00 & 143.59 & \textbf{64.79} & 104.25
& \underline{52.63} & 83.09 & 77.32 & 91.42
& 26.89 & 36.61 & 72.34 & 22.23 & 214.82 \\
LLMLingua v1 & \xmark
& 5.36 & 131.00 & 159.75 & 99.81
& 24.21 & 259.99 & 319.34 & 205.51
& 14.39 & 20.55 & 77.45 & 36.14 & 102.08 \\
LLMLingua-2 & \xmark
& 7.14 & 225.53 & 3{,}364.23 & \underline{83.22}
& 44.21 & 69.59 & 961.40 & 100.50
& 24.12 & 33.56 & 36.92 & 21.13 & \underline{80.77} \\
ACON UT & \xmark
& 29.17 & 261.91 & 1{,}064.69 & 257.49
& 49.47 & 163.11 & 1{,}180.93 & 135.36
& 27.53 & 37.34 & 51.13 & 13.34 & 173.25 \\
ACON UTCO & \xmark
& 30.36 & 243.82 & 1{,}130.67 & 255.95
& 48.42 & 80.77 & 1{,}197.57 & 90.43
& 28.54 & 37.40 & \underline{31.94} & \underline{12.18} & 115.34 \\
\addlinespace[1pt]
\multicolumn{15}{@{}l}{\textit{Soft-token compression}} \\
LCLM & \xmark
& 6.55 & \textbf{23.94} & 1{,}643.65 & \textbf{35.08}
& 21.05 & \textbf{20.83} & 641.30 & \textbf{33.37}
& 0.00 & 0.00 & \textbf{22.47} & \textbf{0.83} & \textbf{0.20} \\
\rowcolor{LabPanel}
\textbf{ASTC (Tools)} & \cmark
& \textbf{33.33} & 81.37 & \underline{80.56} & 94.66
& \textbf{54.74} & 42.09 & \textbf{40.16} & 61.24
& \underline{30.81} & \underline{39.18} & 57.96 & 74.29 & 82.02 \\
\rowcolor{LabPanel}
\textbf{ASTC (Context)} & \cmark
& \underline{32.14} & \underline{80.52} & 81.83 & 95.73
& 51.58 & \underline{37.11} & \underline{40.76} & \underline{58.91}
& \textbf{31.31} & \textbf{41.40} & 66.83 & 77.75 & 87.40 \\
\bottomrule
\end{tabular}
}
\vspace{2pt}
\parbox{\textwidth}{\footnotesize\emph{Note.} Workload includes auxiliary
compressor calls, including LLMLingua, and is relative to Base (100\%).
\textbf{Bold} and \underline{underlined} mark the best and second-best values
by arrow.}
\end{table}

Table~\ref{tab:main-results} shows that ASTC (Tools) offers a favorable balance of task
quality and inference workload under the distinct interaction demands of the
three benchmarks. On AppWorld, it is
the only evaluated method that improves accuracy over Base while reducing all
three workload measures; competing compressors obtain isolated savings but
sacrifice accuracy or increase work elsewhere. ASTC (Context) matches Base's
AppWorld accuracy while reducing all three workload measures. On OfficeBench,
ASTC (Tools) achieves the highest accuracy and the lowest Attention Pairs,
while ASTC (Context) reduces Decode KV and generated tokens further. Both
configurations exceed Base accuracy.
On 8-Objective QA, ASTC (Context) achieves the best EM and F1, while
ASTC (Tools) uses fewer Decode KV interactions, Attention Pairs, and generated
tokens. Both configurations improve answer quality over Base while reducing
all three workload measures. LCLM has the lowest measured workload but answers none of
the tasks correctly; ACON UTCO reduces Decode KV and Attention Pairs further
than ASTC but has lower answer quality. Taken together, these results place
ASTC's value at the trajectory level, where retained task quality and
cumulative inference workload must be judged together.

\subsection{Model-Native Selector Controls}
\label{sec:selector-controls}

We ask whether ASTC's GDN signals improve selection beyond what protection and
sequence shortening provide. We test this by retaining the acting model,
protection rules, and legal group construction while changing the selection
rule. Hidden-only uses hidden-state reconstruction error to test whether local
fidelity suffices; Random assigns fixed random scores to test whether
protected grouping alone explains the gains. Each policy runs the full
AppWorld in closed loop, allowing compression choices to
shape later actions, tool observations, and workload. Figure~\ref{fig:selector-controls}
compares task success with end-to-end Attention Pairs and shows ASTC's
strongest joint outcome, consistent with the value of model-native selection.
Because trajectories and compression events differ, the comparison evaluates
complete policies; paired bootstrap uncertainty makes the success advantage
suggestive rather than conclusive. The comparison suggests that the acting
model's accumulated state supplies information about compressibility beyond
local reconstruction.

\begingroup
\setlength{\intextsep}{6pt}
\begin{figure}[H]
    \centering
    \includegraphics[width=\textwidth]{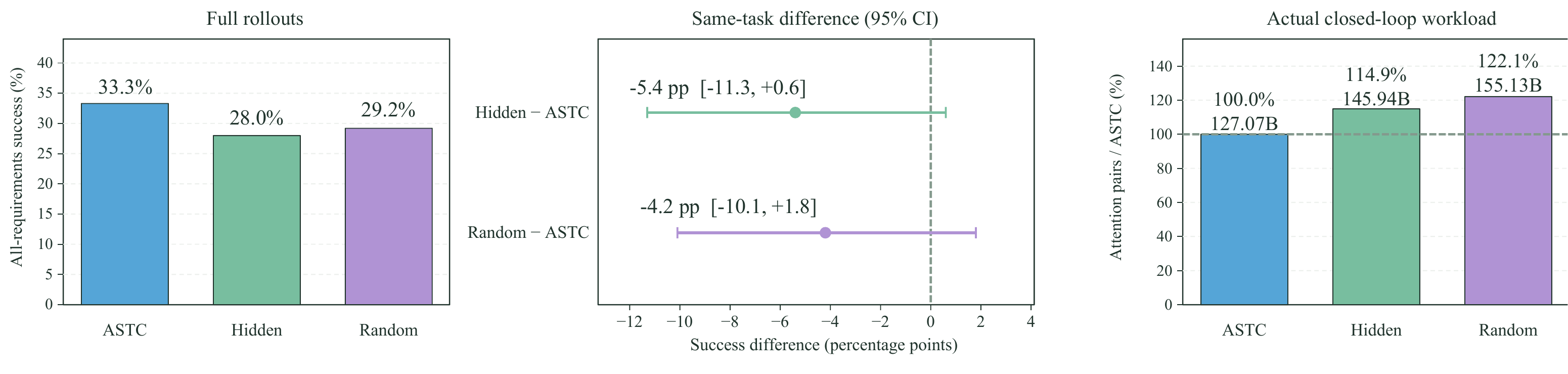}
    \caption{AppWorld selector controls: task success (left), paired success
    differences (middle), and end-to-end Attention Pairs (right).}
    \label{fig:selector-controls}
\end{figure}
\endgroup

\subsection{Persistence After One Commit}
\label{sec:one-commit}

Compression helps a long-horizon agent only if its shorter KV cache remains
useful when later steps revisit it, without substantially changing the
agent's decisions. To isolate this mechanism, we branch from archived Base
trajectories, allow at most one ASTC commit in one branch, and replay identical
subsequent action tokens and tool results in both. This fixed continuation
separates the effect of persistent cache compression from changes caused by
diverging actions or observations. With no further commits allowed, savings at
later horizons reflect reuse of the initial cache update rather than repeated
compression. At successive action horizons during the shared continuation, we measure
KL divergence between conditional action-token distributions and cumulative
Attention Pairs saved since the commit. Figure~\ref{fig:one-commit-reuse}
shows that the distributions remain close while savings recur at later
actions, supporting the proposed amortization across the trajectory. A
separate autonomous one-commit run probes whether this local stability carries
through to task outcomes; its success is somewhat lower than archived Base
and full ASTC, and differing archive versions make the comparison descriptive.

\begingroup
\setlength{\intextsep}{6pt}
\begin{figure}[H]
    \centering
    \includegraphics[width=\textwidth]{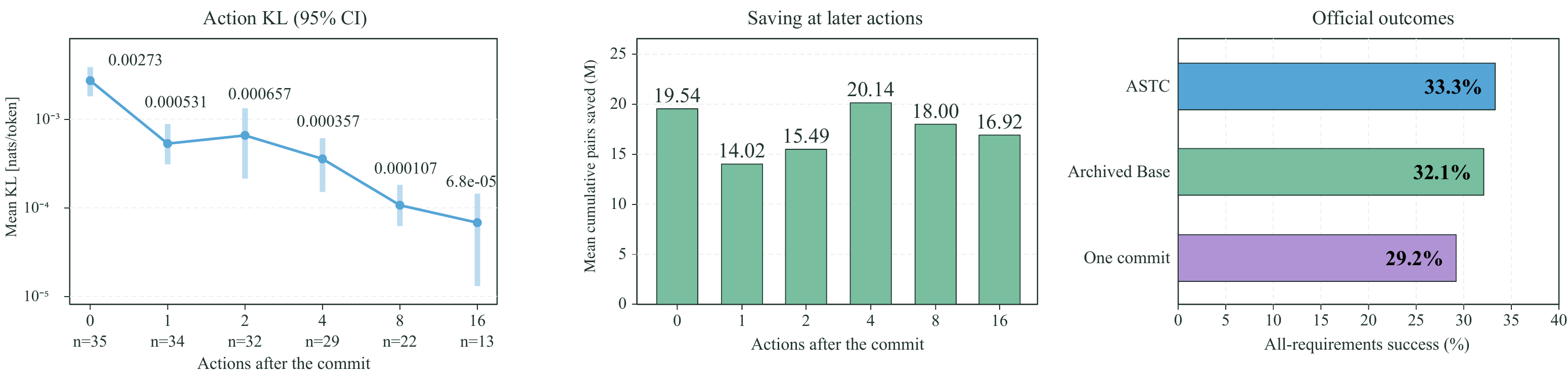}
    \caption{One-commit cache reuse on AppWorld: action-distribution KL
    divergence (left), saved Attention Pairs (middle), and autonomous task
    success (right).}
    \label{fig:one-commit-reuse}
\end{figure}
\endgroup

\subsection{Component Ablations}

\begin{wraptable}{r}{0.54\textwidth}
\vspace{-6pt}
\caption{Component ablations on AppWorld.}
\label{tab:component-ablation}
\centering
\small
\setlength{\tabcolsep}{1.5pt}
\renewcommand{\arraystretch}{1.08}
\begin{tabular*}{\linewidth}{@{\extracolsep{\fill}}lcccc@{}}
\toprule
Variant
& Acc.\% $\uparrow$ & D-KV\% $\downarrow$ & Attn.\% $\downarrow$ & Gen.\% $\downarrow$ \\
\midrule
Base
& 32.14 & 100.00 & 100.00 & 100.00 \\
Prompting
& 25.00 & 143.59 & \textbf{64.79} & 104.25 \\
ASTC
& \textbf{33.33} & \textbf{81.55} & 80.70 & \textbf{94.83} \\
No protection
& 20.83 & 98.46 & 68.63 & 127.05 \\
No $\chi$
& 32.14 & 84.79 & 78.11 & 97.76 \\
No $\psi$
& 32.74 & 89.72 & 86.08 & 99.65 \\
ASTC (deletion)
& 32.74 & 88.83 & 86.85 & 99.58 \\
\bottomrule
\end{tabular*}
\vspace{2pt}
\parbox{\linewidth}{\footnotesize Prompting is a deletion reference reported
in Table~\ref{tab:main-results}.}
\vspace{-6pt}
\end{wraptable}

Table~\ref{tab:component-ablation} includes \textbf{ASTC (deletion)}, which
keeps the protection mask and both GDN selection tests but removes selected
positions without emitting soft tokens. Its accuracy remains close to that
of full ASTC, benefiting from the same model-native selection mechanism.
In contrast,
\textbf{Prompting} applies conventional hard-token deletion without ASTC's
selector and achieves only 25.00\% accuracy. Removing protection
sharply lowers accuracy to 20.83\% despite reducing Attention Pairs. Losing
API contracts, identifiers, or other exact values can invalidate later actions
and lengthen recovery trajectories, explaining why these values must remain
explicit. Removing either selection test also lowers accuracy and raises
Decode KV and Generated Tokens relative to full ASTC. These ablations show
how model-native selection, exact-value protection, and
soft-token retention work together to preserve useful context while reducing
inference work.

\subsection{Cross-Model Transfer}
\label{sec:cross-model-transfer}
\label{sec:model-transfer}

\begin{wraptable}{r}{0.60\textwidth}
\vspace{-6pt}
\caption{Cross-model transfer on AppWorld.}
\label{tab:cross-model-transfer}
\label{tab:model-transfer}
\centering
\small
\setlength{\tabcolsep}{1.5pt}
\renewcommand{\arraystretch}{1.08}
\begin{tabular*}{\linewidth}{@{\extracolsep{\fill}}llccc@{}}
\toprule
Model & Method & Acc.\% $\uparrow$ & D-KV\% $\downarrow$ & Attn.\% $\downarrow$ \\
\midrule
\multirow{2}{*}{Qwen3.5-9B} & Base & 32.14 & 100.00 & 100.00 \\
& ASTC & \textbf{33.33} & \textbf{81.55} & \textbf{80.70} \\
\addlinespace[2pt]
\multirow{2}{*}{Qwen3.8-27B} & Base & \textbf{78.57} & 100.00 & 100.00 \\
& ASTC & 77.98 & \textbf{91.70} & \textbf{64.43} \\
\addlinespace[2pt]
\multirow{2}{*}{Kimi-Linear-48B-A3B} & Base & 1.47 & 100.00 & 100.00 \\
& ASTC & \textbf{4.41} & \textbf{84.01} & \textbf{88.76} \\
\bottomrule
\end{tabular*}
\vspace{-6pt}
\end{wraptable}

Table~\ref{tab:cross-model-transfer} evaluates cross-model transfer on
AppWorld, with workload normalized to each model's Base. ASTC reduces Decode
KV and Attention Pairs for all three models. On Qwen3.8-27B, Attention Pairs
decrease by 35.57\% with an accuracy decrease of 0.59 percentage points,
while the other two models improve accuracy. These results support the
transferability of ASTC's training-free compression criterion; additional
details are provided in Appendix~\ref{app:model-transfer}.

\section{Conclusion}
\label{sec:conclusion}

Long-horizon agents must retain evidence for future decisions as context grows,
while separately trained compressors may drift from the acting LLM's evolving
task representation. ASTC uses native recurrent dynamics after early layers
have read the full sequence, carrying absorbed content as agentic soft tokens
and keeping exact values as hard tokens. The shorter sequence cuts later-layer
computation and creates a compact KV cache whose savings persist across
subsequent decisions. Without compressor training, ASTC keeps accuracy close
to or above the uncompressed agent with up to 59.84\% fewer end-to-end Attention
Pairs; in the controlled comparison, it attains the highest accuracy and
lowest Attention Pairs per successful task. Ablations and fixed-continuation
analysis link these gains to exact-value protection, model-native selection,
and persistent cache reuse. These results show how the acting LLM can compact
absorbed context into reusable state without losing exact evidence for later
decisions.

\subsection*{AI Use Statement}

Generative AI tools were used to polish the language and presentation of
author-provided manuscript text, including rephrasing and reorganizing sections
for clarity. The authors reviewed all AI-assisted edits and take full
responsibility for the paper's content.

\begingroup
\small
\setlength{\bibsep}{0pt}
\bibliography{references}
\bibliographystyle{plainnat}
\endgroup

\clearpage
\appendix
\addtocontents{toc}{\protect\ASTCAppendixTocStart}
\begingroup
\newif\ifASTCAppendixEntry
\ASTCAppendixEntryfalse
\newcommand{\ASTCAppendixTocStart}{\ASTCAppendixEntrytrue}
\let\ASTCOriginalContentsLine\contentsline
\renewcommand{\contentsline}[4]{%
  \ifASTCAppendixEntry
    \ASTCOriginalContentsLine{#1}{#2}{#3}{#4}%
  \fi
}
\renewcommand{\contentsname}{Appendix Contents}
\setcounter{tocdepth}{2}
\tableofcontents
\endgroup
\clearpage

\section{Theoretical Derivations}
\label{app:theory-derivations}

The distributional comparisons below fix an agent history, a next-token
prefix, and a compression boundary while preserving original logical
positions. They bound one conditional next-token distribution and do not
claim that independently generated trajectories remain identical. The
workload proof separately counts later queries while the compressed cache
persists, and the attention-mass proof fixes one downstream query.

\subsection{Surrogate-Objective Mismatch}

Let $\mathcal A_k$ be the same nonempty feasible set for both objectives,
and assume the minima below exist.  Write $D(Z)=D_\theta(H,Z)$ and
$J(Z)=J_{\mathrm{tr}}(H,Z)$.  If
$\sup_{Z\in\mathcal A_k}|D(Z)-J(Z)|\leq\varepsilon$ and
$Z_{\mathrm{tr}}\in\mathcal A_k$ satisfies
$J(Z_{\mathrm{tr}})\leq\inf_{Z\in\mathcal A_k}J(Z)+\xi$, then, for
$Z_\theta^\star\in\arg\min_{Z\in\mathcal A_k}D(Z)$,
\begin{equation}
\begin{aligned}
D(Z_{\mathrm{tr}})
&\leq J(Z_{\mathrm{tr}})+\varepsilon\\
&\leq J(Z_\theta^\star)+\xi+\varepsilon\\
&\leq D(Z_\theta^\star)+2\varepsilon+\xi .
\end{aligned}
\label{eq:app-surrogate-bound}
\end{equation}
Optimality of $Z_\theta^\star$ gives the nonnegative lower endpoint.
If $J=D$ and $\xi=0$, a trained method can attain the optimum. Thus the
upper bound separates objective mismatch from optimization error but gives
no lower bound on the realized gap and does not order ASTC and trained
methods. It also does not apply across unequal feasible sets,
such as unrestricted soft states and text-reachable hard states.

\subsection{Model-Native Selection Bounds}

Write $D_v=\sum_h d_{v,h}$ for the total value width across GDN heads.
For each head $h$, the update
$\mathbf S_{i,h}=\alpha_{i,h}\mathbf S_{i-1,h}
+\beta_{i,h}\boldsymbol\delta_{i,h}\mathbf k_{i,h}^\top$ with
$\beta_{i,h}\geq0$ gives the rank-one identity
\begin{equation}
\|\mathbf S_{i,h}-\alpha_{i,h}\mathbf S_{i-1,h}\|_F
=\beta_{i,h}\|\boldsymbol\delta_{i,h}\|_2\|\mathbf k_{i,h}\|_2.
\end{equation}
The implemented RMS aggregates values across heads. Hence
$\|\boldsymbol\delta_{i,h}\|_2\leq\|\boldsymbol\delta_i\|_2
=\sqrt{D_v}\,\chi_i(\operatorname{RMS}(\mathbf v_i)+\epsilon)$.
Accepting a position with $\chi_i<\tau_\chi$ and $\psi_i<\tau_\psi$
therefore implies
\begin{equation}
\begin{aligned}
\|\mathbf S_{i,h}-\alpha_{i,h}\mathbf S_{i-1,h}\|_F
&<\beta_{i,h}\tau_\chi\sqrt{D_v}
(\operatorname{RMS}(\mathbf v_i)+\epsilon)\|\mathbf k_{i,h}\|_2,\\
\|\mathbf v_i-\widehat{\mathbf v}_i\|_2
&<\tau_\psi\sqrt{D_v}
(\operatorname{RMS}(\mathbf v_i)+\epsilon).
\end{aligned}
\label{eq:app-native-bounds}
\end{equation}
Both are local to the GDN state and the specified endpoint reader.  They
do not constrain the attention scores of the same token.  Pre-boundary
layers process every original token and preserve their recurrent state,
so, conditional on the same future input tokens, that state is unchanged
by the boundary replacement.  The remaining possible perturbation enters
through the compressed downstream sequence and its attention caches.

\subsection{Downstream Attention Stability}

Consider one attention head, the same fixed query vector in both runs after
the greatest logical position of a candidate group $G$, and its replacement.
Assume at least one retained position exists and all retained keys and
values are unchanged in this local comparison. Let $R$ be the other positions visible
to this query, and define
$W_G=\sum_{i\in G}e^{s_i}$, $W_R=\sum_{i\in R}e^{s_i}$, and
$w_z=e^{s_z}$, where $s_i$ and $s_z$ are the original and replacement
attention scores for this query.  Define
\begin{equation}
\begin{gathered}
a_G=\frac{W_G}{W_R+W_G},\qquad
b_G=\frac{w_z}{W_R+w_z},\\
\bar{\mathbf v}_G=\frac{\sum_{i\in G}e^{s_i}\mathbf v_i}{W_G},
\qquad
\bar{\mathbf v}_R=\frac{\sum_{i\in R}e^{s_i}\mathbf v_i}{W_R}.
\end{gathered}
\end{equation}
The two attention outputs and their difference are exactly
\begin{equation}
\begin{aligned}
\mathbf o&=(1-a_G)\bar{\mathbf v}_R+a_G\bar{\mathbf v}_G,\\
\widetilde{\mathbf o}
&=(1-b_G)\bar{\mathbf v}_R+b_G\mathbf v_z,\\
\widetilde{\mathbf o}-\mathbf o
&=(b_G-a_G)(\mathbf v_z-\bar{\mathbf v}_R)
  +a_G(\mathbf v_z-\bar{\mathbf v}_G).
\end{aligned}
\label{eq:app-attention-identity}
\end{equation}
The triangle inequality gives
$\|\widetilde{\mathbf o}-\mathbf o\|_2\leq
|b_G-a_G|\|\mathbf v_z-\bar{\mathbf v}_R\|_2
+a_G\|\mathbf v_z-\bar{\mathbf v}_G\|_2$.
If the value map is $L_V$-Lipschitz on the group and
its representative, and every member is within hidden-state radius
$r_G$ of the representative, the second term is at most
$a_G L_Vr_G$; the attention-mass term remains separate.

The attention-mass condition cannot be omitted.  As a counterexample,
take $g>1$ identical group values $\mathbf v_G$, one retained value
$\mathbf v_R\ne\mathbf v_G$, and equal attention scores.  If the soft
representative has value $\mathbf v_G$ and the same score, then
$a_G=g/(g+1)$ but $b_G=1/2$, so the two outputs differ despite zero
within-group reconstruction error.  The GDN innovation and endpoint
read error can also be zero in a hybrid model whose attention projections
are independent of that state.  Thus the native tests alone cannot certify
unchanged downstream decisions.

\subsection{Downstream Recurrent Stability}

Post-boundary GDN layers also consume the shortened sequence.  Consider
one such layer and align the original and compressed runs at the logical
endpoint $t_G$ of a group $G$. Starting from the same state, let their states
there differ by $e_G$. For later tokens that supply identical keys,
values, and gates to this layer, the GDN recurrence has the affine form
\begin{equation}
\mathbf S_t=\alpha_t\mathbf S_{t-1}
(\mathbf I-\beta_t\mathbf k_t\mathbf k_t^\top)
+\beta_t\mathbf v_t\mathbf k_t^\top .
\end{equation}
Its state map is $\rho_t$-Lipschitz in Frobenius norm with
$\rho_t=\alpha_t\|\mathbf I-\beta_t\mathbf k_t\mathbf k_t^\top\|_2$.
Induction therefore gives, after $m$ such tokens,
\begin{equation}
\|\mathbf S_{t_G+m}-\widetilde{\mathbf S}_{t_G+m}\|_F
\leq e_G\prod_{u=1}^{m}\rho_{t_G+u}.
\label{eq:app-state-propagation}
\end{equation}
If $\|\mathbf k_t\|_2\leq1$, $0\leq\beta_t\leq1$, and
$0\leq\alpha_t\leq1$, then $\rho_t\leq\alpha_t\leq1$.
This result is conditional on identical later inputs to that GDN layer.
It does not bound the initial defect $e_G$, which concerns a different
downstream layer from the GDN used to compute $\chi$ and $\psi$.

For a full-model statement, assume an alignment between the two finite
causal computations at every group endpoint.  Assume further that they
can be joined by a sequence of hybrids in which each adjacent pair
differs at one affected attention output or downstream recurrent state,
with all its upstream inputs equal.  Let $e_j$ bound that local
difference (using the attention bound above or a state-defect bound),
and let the map from that quantity to the next-token logits be
$L_j$-Lipschitz on the segment between the two values.  Applying the
triangle inequality to the successive hybrid logits proves
\begin{equation}
\|\widetilde{\mathbf l}-\mathbf l\|_\infty
\leq E:=\sum_jL_je_j .
\label{eq:app-logit-bound}
\end{equation}
Because log-sum-exp is 1-Lipschitz in the infinity norm, each
log-probability changes by at most $2E$ and
$D_{\mathrm{KL}}(p\|\widetilde p)\leq2E$.  If the original largest
logit exceeds the second largest by more than $2E$, its index remains
the greedy next token.  This theorem is conditional: $\chi$, $\psi$,
and the pooling threshold alone do not supply all $e_j$ or $L_j$.
It concerns the fixed next-token prefix and does not extend to
independently generated trajectories without further stability
assumptions.

\subsection{Continuous Representation at a Fixed Budget}

Fix $k$ and logical position IDs $\pi$ at the compression boundary;
hold the cached prefix and all untouched boundary states fixed. The notation
$F_{0:\ell^\star}(t_{1:k};\pi)$ evaluates all candidate sequences at those
same positions. For a finite vocabulary $\mathcal V$, the hard-token
reachable set
$\mathcal A_k^{\mathrm{hard}}
=\{F_{0:\ell^\star}(t_{1:k};\pi):t_{1:k}\in\mathcal V^k\}$
contains at most $|\mathcal V|^k$ states. In the idealized continuous
comparison, $\mathcal A_k^{\mathrm{soft}}=\mathbb R^{k\times d}$
contains this set. For any distortion $D_\theta$ defined on both sets,
\begin{equation}
\inf_{Z\in\mathcal A_k^{\mathrm{soft}}}D_\theta(H,Z)
\leq\min_{Z\in\mathcal A_k^{\mathrm{hard}}}D_\theta(H,Z).
\label{eq:app-soft-inclusion}
\end{equation}
This inclusion alone cannot establish a strict advantage: the acting
model might assign the same next-token distribution to a reachable hard
state and to the full sequence.

Let $p_H$ be the full-context next-token distribution and
$\delta_k(H)=\min_{Z\in\mathcal A_k^{\mathrm{hard}}}
\|p_H-p_\theta^{>\ell^\star}(\cdot\mid Z)\|_1$.
For any two categorical distributions $p,q$, Pinsker's inequality in
natural-log units is
\begin{equation}
D_{\mathrm{KL}}(p\|q)\geq\tfrac12\|p-q\|_1^2.
\label{eq:app-pinsker}
\end{equation}
For completeness, take $A=\{i:p_i\geq q_i\}$, $a=p(A)$, and $b=q(A)$.
Then $\|p-q\|_1=2(a-b)$. Grouping outcomes into $A$ and its complement
cannot increase KL, while the binary KL as a function of $a$ has
second derivative $1/[a(1-a)]\geq4$ and vanishes with zero derivative
at $a=b$. Hence binary KL is at least $2(a-b)^2$, which proves
Equation~\ref{eq:app-pinsker}. It follows that
\begin{equation}
\min_{Z\in\mathcal A_k^{\mathrm{hard}}}D_\theta(H,Z)
\geq\tfrac12\delta_k(H)^2.
\label{eq:app-hard-lower-bound}
\end{equation}
If $\delta_k(H)>0$, no hard-token sequence of length $k$ reproduces
the full-context next-token distribution exactly. The finiteness of
the hard output set makes this condition equivalent to $p_H$ being
absent from that set.

For ASTC's particular output $Z_{\mathrm{ASTC}}$, the conditional
logit bound in Equation~\ref{eq:app-logit-bound} gives
$D_\theta(H,Z_{\mathrm{ASTC}})\leq2E$. Therefore
\begin{equation}
E<\delta_k(H)^2/4
\quad\Longrightarrow\quad
D_\theta(H,Z_{\mathrm{ASTC}})
<\min_{Z\in\mathcal A_k^{\mathrm{hard}}}D_\theta(H,Z).
\label{eq:app-strict-soft-advantage}
\end{equation}
The statement compares equal physical length, fixed positions, and one
next-token distribution. Its premise requires an actual bound on all
downstream attention and recurrent defects. The ambient continuous set
is larger than ASTC's attainable mean-pooled set, and finite-precision
execution further restricts it. Thus the set inclusion establishes
capacity, while Equation~\ref{eq:app-strict-soft-advantage} specifies
when one ASTC representation realizes a strict advantage.

For completeness, the pre-cast representative in
Equation~\ref{eq:native-soft-token} is the least-squares solution. With
$\bar{\mathbf h}_G=|G|^{-1}\sum_{i\in G}\mathbf h_i$, the identity
\begin{equation}
\sum_{i\in G}\|\mathbf h_i-\mathbf z\|_2^2
=\sum_{i\in G}\|\mathbf h_i-\bar{\mathbf h}_G\|_2^2
+|G|\|\mathbf z-\bar{\mathbf h}_G\|_2^2
\end{equation}
shows that the FP32 mean minimizes squared hidden-state error. This does
not minimize downstream KL; ASTC checks reconstruction again after casting.

\subsection{Proof of Theorem~\ref{thm:amortized-attention}}
\label{app:amortized-proof}

\begin{proof}
Let $p$ be the unchanged physical length of the cached prefix and $m$ the
number of newly processed positions before the commit, so
$n_{\mathrm{phys}}=p+m$ and $0<r<m$. At one full causal attention layer,
the new queries create
\begin{equation}
Q_{\mathrm{before}}=pm+\frac{m(m+1)}{2}
\end{equation}
query--key pairs: each attends the $p$ cached keys and its causal prefix in
the new block. After replacing $r$ new positions, the unchanged cached
prefix still contributes $p$ keys, but only $m-r$ new queries remain. Thus
\begin{equation}
Q_{\mathrm{after}}
=p(m-r)+\frac{(m-r)(m-r+1)}{2}.
\end{equation}
Subtracting gives the same-pass reduction
\begin{equation}
Q_{\mathrm{before}}-Q_{\mathrm{after}}
=pr+\frac{r(2m-r+1)}{2}
=\frac{r(2n_{\mathrm{phys}}-r+1)}{2}.
\end{equation}
The fresh-prefill case is $p=0$. Every later query that actually reaches
this attention layer under the same fixed continuation has exactly $r$
fewer historical keys while the compared caches retain their length
difference; $N_q$ such queries add $rN_q$ saved pairs. Summing
over the $A_{>}$ affected full-attention layers proves
Equation~\ref{eq:amortized-attention}. The group positions need not be
contiguous because the count depends on the physical length after stable
sorting, not on which original positions were merged. The result excludes
work already spent on the cached prefix and assumes no intervening change
to the compared cache that alters this $r$-position difference. A future
token compressed before reaching these layers is not counted in $N_q$.
\end{proof}

\subsection{Proof of Lemma~\ref{lem:attention-refinement}}
\label{app:attention-refinement-proof}

\begin{proof}
The query follows the greatest logical position of $G$, so all $g$ original
members are visible in the uncompressed run and their one representative
is visible in the compressed run. Their identical scores give group
unnormalized masses $ge^s$ and $e^s$, respectively. Because the retained
positions are unchanged, their attention masses are
\begin{equation}
A_R=\frac{Z_R}{Z_R+ge^s},\qquad
\widetilde A_R=\frac{Z_R}{Z_R+e^s}.
\end{equation}
Direct subtraction yields
\begin{equation}
\widetilde A_R-A_R
=\frac{Z_R(g-1)e^s}{(Z_R+e^s)(Z_R+ge^s)}>0,
\end{equation}
since $g>1$ and $Z_R>0$. The two attention outputs are
\begin{equation}
\mathbf o=A_R\overline{\mathbf v}_R+(1-A_R)\mathbf v_G,
\qquad
\widetilde{\mathbf o}
=\widetilde A_R\overline{\mathbf v}_R
+(1-\widetilde A_R)\mathbf v_G.
\end{equation}
Subtracting and taking the inner product with any $\mathbf u$ gives
Equation~\ref{eq:attention-refinement}. Positivity of the projected
change requires $\mathbf u^\top(\overline{\mathbf v}_R-\mathbf v_G)>0$;
it is not a claim about arbitrary nonlinear decisions. If scores or
values differ, the exact identity in Equation~\ref{eq:app-attention-identity}
separates attention-mass shift from value error.
\end{proof}

\clearpage
\section{Agent Trajectory Examples}
\label{app:trajectory-examples}

Figure~\ref{fig:agent-trajectory-demo} shows two AppWorld trajectories.
ASTC retains exact fields as hard tokens and replaces selected history
positions with soft tokens. The reported attention-pair reductions describe
logical computation in these trajectories, not measured wall-clock speedups.

\begin{figure}[H]
\centering
\includegraphics[width=\textwidth,height=0.82\textheight,keepaspectratio]{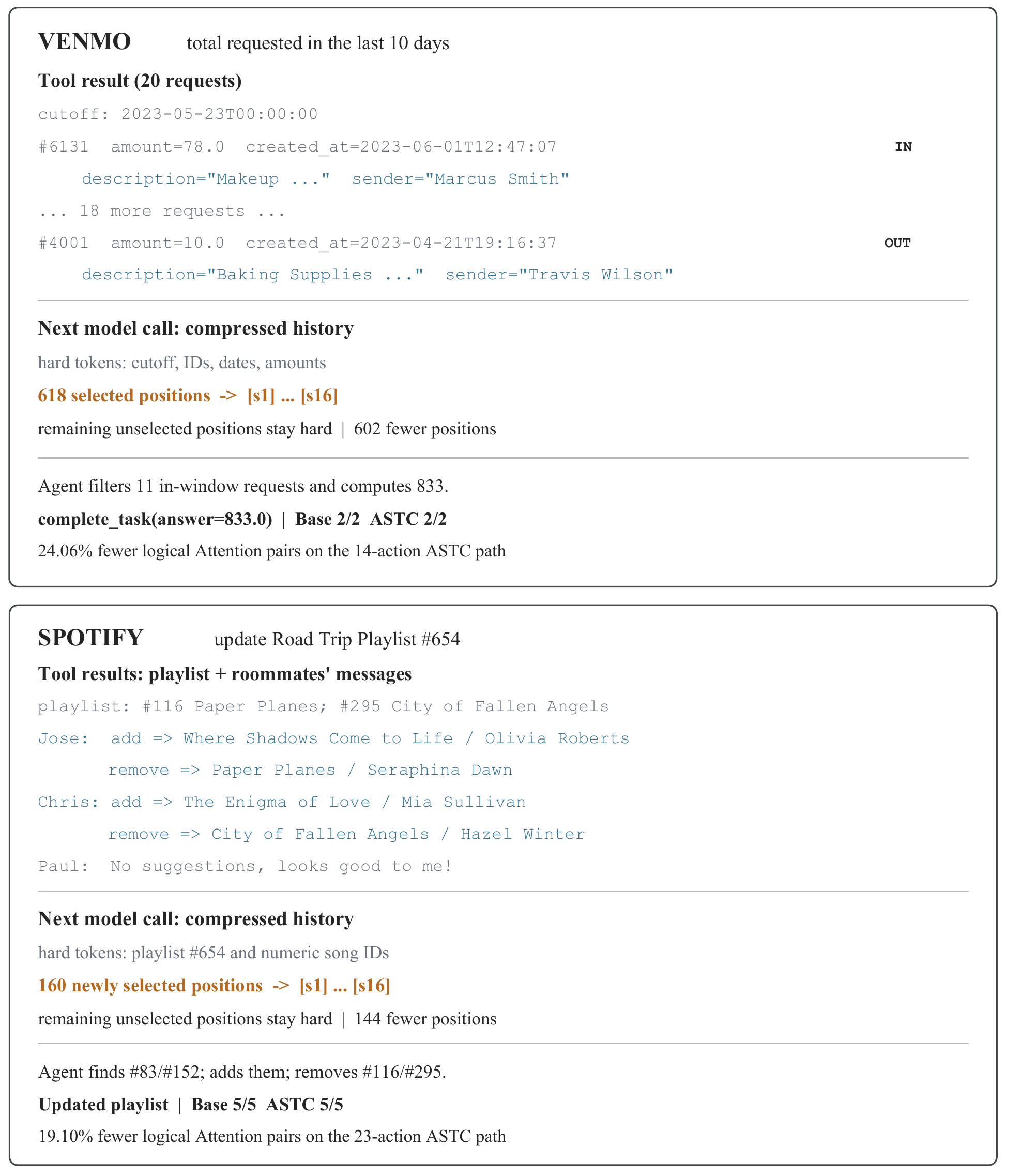}
\caption{ASTC compression in Venmo and Spotify agent trajectories.}
\label{fig:agent-trajectory-demo}
\end{figure}

\clearpage
\section{Component-Ablation Details}
\label{app:ablation-details}

Each ASTC variant in Table~\ref{tab:component-ablation} changes one component
while keeping the frozen model and inference configuration fixed.
\textbf{ASTC (deletion)} retains the protection mask and both GDN selection
tests, then deletes ASTC-selected positions instead of emitting soft tokens.
It tests the retained representation with model-native selection still in
place. \textbf{Prompting} is reproduced from Table~\ref{tab:main-results} as a
separate prompt-based text-level deletion reference. Prompting compresses the
full context in AppWorld, whereas the ASTC ablations target eligible tool
observations. Their comparison therefore does not isolate the effect of
selection alone; the remaining ASTC variants provide the matched component
comparisons.

Removing protection lowers accuracy from 33.33\% to 20.83\% (21 fewer
successes), even as Attention Pairs fall to 68.63\% of Base; Generated Tokens
rise to 127.05\%. Losing API contracts, identifiers, or exact values can
invalidate later actions and lengthen recovery trajectories.

Without $\chi$, accuracy returns to Base's 32.14\%, two fewer successes
than full ASTC. Attention Pairs fall to 78.11\%, but Decode KV and Generated
Tokens rise relative to full ASTC. Without $\psi$, accuracy is 32.74\% and
all three workload measures rise. The tests address distinct cases: new
writes to recurrent state and values no longer readable at the endpoint.

ASTC (deletion) achieves 32.74\% accuracy, compared with 25.00\% for
Prompting and 33.33\% for full ASTC. Its strong performance already benefits
from ASTC's model-native selection. Relative to full ASTC, it loses one
success and raises Decode KV, Attention Pairs, and Generated Tokens from
81.55\%, 80.70\%, and 94.83\% to 88.83\%, 86.85\%, and 99.58\%,
respectively. This comparison measures the additional contribution of
retaining the selected information as soft tokens.

\section{Additional Dataset Experiments}
\label{app:additional-datasets}

We evaluate ASTC with the frozen Qwen3.5-9B model on RULER 32K
\citep{hsieh2024ruler}, LongBench v2~\citep{bai2024longbenchv2}, and
interactive code repair on MBPP~\citep{austin2021mbpp}. The respective sample
counts are 260, 503, and 427. All comparisons use a matched Base run within
each benchmark; the protocols are detailed in
Appendix~\ref{app:additional-protocols}.

\begin{table}[H]
\caption{Additional-dataset results.}
\label{tab:additional-datasets}
\centering
\setlength{\tabcolsep}{2pt}
\renewcommand{\arraystretch}{1.08}
\resizebox{\textwidth}{!}{%
\begin{tabular}{@{}l*{9}{c}@{}}
\toprule
& \multicolumn{3}{c}{\textbf{RULER 32K}}
& \multicolumn{3}{c}{\textbf{LongBench v2}}
& \multicolumn{3}{c}{\textbf{MBPP}} \\
\cmidrule(lr){2-4}\cmidrule(lr){5-7}\cmidrule(l){8-10}
Method
& Acc.\% $\uparrow$ & D-KV\% $\downarrow$ & Attn.\% $\downarrow$
& Acc.\% $\uparrow$ & D-KV\% $\downarrow$ & Cache TB $\downarrow$
& Pass\% $\uparrow$ & D-KV\% $\downarrow$ & Attn.\% $\downarrow$ \\
\midrule
Base & 97.173 & 100.00 & 100.00
     & 40.95 & 100.00 & 2.065
     & \textbf{84.78} & 100.00 & 100.00 \\
\rowcolor{LabPanel}
\textbf{ASTC} & 97.173 & \textbf{93.96} & \textbf{85.57}
     & 40.95 & \textbf{92.93} & \textbf{1.921}
     & 84.07 & \textbf{92.42} & \textbf{92.45} \\
\bottomrule
\end{tabular}%
}
\end{table}

On RULER, ASTC matches Base's aggregate answer score to the reported
precision while reducing Decode KV by 6.04\% and Attention Pairs by 14.43\%.
It migrates 1,045,409 of 6,485,833 eligible tokens into 1,468 soft tokens,
for a 12.56\% net prompt-token reduction. This 20-example-per-task run is
smaller than RULER's 500-example reference configuration.

LongBench v2 yields 206/503 correct for both variants, with four Base-only
and four ASTC-only successes. ASTC reduces Decode KV visits by 7.07\%,
all-layer final KV tokens by 6.19\%, and analytical KV-cache bytes by 6.98\%.
The subgroup changes are mixed (Base/ASTC: Easy 44.79/45.83\%, Hard
38.59/37.94\%; Short 44.44/45.56\%, Medium 40.00/39.53\%, Long
37.04/36.11\%). The 262,144-token limit truncates 100 examples in each
variant; these figures do not establish a wall-clock speedup.

On MBPP, ASTC compresses at least one span in 182/427 trajectories and
reduces workload; its paired success difference is $-0.70$ percentage points
(95\% bootstrap interval $[-1.87,0.23]$). Mean generated tokens per problem
are 89.2/89.4 on RULER, 2.0/2.0 on LongBench, and 337.4/329.0 on MBPP
(Base/ASTC). The MBPP run does not establish a success-rate improvement.

\clearpage
\section{Cross-Model Transfer}
\label{app:model-transfer}

For the Qwen3.5-9B and Kimi-Linear comparisons in
Section~\ref{sec:cross-model-transfer}, we use the following model-specific
configurations. Each frozen model uses a reader
calibrated for its own internal state and a matching uncompressed Base run. ASTC compresses eligible tool
observations while keeping API contracts and authentication content explicit.
Qwen3.5-9B uses Stage 3, while Kimi-Linear uses Stage 5. Within each model,
Base and ASTC share task IDs and decoding settings; these settings need not
match across models. Kimi-Linear covers the first 68 Test-Normal tasks in
official order; Qwen3.5-9B covers all 168.

On Kimi-Linear, ASTC reduces Decode KV, Attention Pairs, and generated tokens.
Successes rise from one to three on its 68-task subset. These Kimi-Linear results do not establish a general accuracy gain
across backbones.


\section{Soft Aggregation and Subsequent Attention Reads}
\label{app:soft-representation-controls}

We test the value of soft aggregation and whether later attention uses the
resulting representatives on all 168 AppWorld Test-Normal tasks, with
Qwen3.5-9B, Stage 3, and Tool compression. The \emph{equal-length hard}
control uses the same groups, representative counts, and representative
positions as ASTC on a common input prefix, but keeps the hidden state of
each group's last original token instead of aggregating its member states.
This isolates the representation choice at a fixed length budget. In a
second control, \emph{read-masked soft} representatives become inaccessible
to attention starting from the next forward pass after their creation.
Their physical KV entries and logical positions remain in place, and GDN
updates follow the original rule. This control tests subsequent attention
access while retaining the recurrent-state path.

\begin{table}[H]
\caption{Soft aggregation and subsequent attention access on AppWorld.}
\label{tab:soft-representation-controls}
\centering
\small
\setlength{\tabcolsep}{3pt}
\renewcommand{\arraystretch}{1.08}
\begin{tabular*}{\linewidth}{@{\extracolsep{\fill}}lccccc@{}}
\toprule
Representation & \shortstack{Later attention\\access}
& Acc.\% $\uparrow$ & D-KV\% $\downarrow$
& Attn.\% $\downarrow$ & Gen.\% $\downarrow$ \\
\midrule
Base & --- & 32.14 & 100.00 & 100.00 & 100.00 \\
Equal-length hard & \cmark & 30.95 & 89.20 & 85.72 & 99.08 \\
ASTC soft & \cmark & \textbf{33.33} & \textbf{81.37}
& \textbf{80.56} & \textbf{94.66} \\
ASTC soft & \xmark & 27.38 & 89.17 & 85.06 & 99.25 \\
\bottomrule
\end{tabular*}
\end{table}

Table~\ref{tab:soft-representation-controls} shows that aggregation reaches
33.33\% accuracy, compared with 30.95\% for a single original state per
group; blocking subsequent attention access lowers accuracy to 27.38\%.
These results support soft tokens as a way to aggregate group information
within a limited length budget and retain a representation useful to later
reasoning. Equal length is enforced on common inputs; compression choices
can still lead to different autonomous trajectories and cumulative workloads.

\section{Measured Runtime and Memory}
\label{app:runtime-memory}

We supplement the analytical workload counts with measured task latency,
throughput, and peak GPU memory for Qwen3.5-9B on AppWorld and 8-Objective QA
(Tables~\ref{tab:runtime-gpu} and~\ref{tab:runtime-host}). Mean task latency
falls by approximately 3.7\% on AppWorld and 13.45\% on QA, and QA's P95
latency falls by 19.31\%. AppWorld successes increase from 54 to 56, while
QA EM/F1 increases from 28.91/37.51 to 30.81/39.18. Thus the task-level
latency reductions accompany the reported task-quality results.

\begin{table}[H]
\caption{Measured task latency, per-GPU throughput, and peak GPU memory.}
\label{tab:runtime-gpu}
\centering
\small
\setlength{\tabcolsep}{3pt}
\renewcommand{\arraystretch}{1.12}
\begin{tabular*}{\linewidth}{@{\extracolsep{\fill}}llcccc@{}}
\toprule
Dataset & Method & \shortstack{Mean latency\\(s/task) $\downarrow$}
& \shortstack{P95 latency\\(s/task) $\downarrow$}
& \shortstack{Token/s/\\GPU $\uparrow$}
& \shortstack{Peak GPU\\memory (GiB) $\downarrow$} \\
\midrule
\multirow{2}{*}{AppWorld} & Base & 73.69 & 210.08
& \textbf{24.87} & 21.68 \\
& ASTC & \textbf{71.00} & \textbf{207.47} & 24.43 & \textbf{21.46} \\
\addlinespace[2pt]
\multirow{2}{*}{8-Objective QA} & Base & 182.11 & 339.27
& \textbf{26.95} & \textbf{34.72} \\
& ASTC & \textbf{157.62} & \textbf{273.77} & 25.54 & 36.97 \\
\bottomrule
\end{tabular*}
\end{table}

\begin{table}[H]
\caption{Queue-level performance and host-side compression timings.}
\label{tab:runtime-host}
\centering
\small
\setlength{\tabcolsep}{2pt}
\renewcommand{\arraystretch}{1.12}
\begin{tabular*}{\linewidth}{@{\extracolsep{\fill}}llccccc@{}}
\toprule
Dataset & Method & \shortstack{Queue time\\(min) $\downarrow$}
& \shortstack{Queue\\token/s $\uparrow$}
& Selection$^{*}$ & Pooling$^{*}$ & \shortstack{Entry\\total$^{*}$} \\
\midrule
\multirow{2}{*}{AppWorld} & Base & 37.70 & 136.10 & 0 & 0 & 0 \\
& ASTC & \textbf{34.43} & \textbf{141.05} & 5.664 & 13.292 & 46.005 \\
\addlinespace[2pt]
\multirow{2}{*}{8-Objective QA} & Base & --- & --- & 0 & 0 & 0 \\
& ASTC & --- & --- & 280.057 & 144.705 & 631.412 \\
\bottomrule
\end{tabular*}
\vspace{2pt}
\parbox{\linewidth}{\footnotesize $^{*}$Per-task average host-side time in
milliseconds. Selection and pooling are components of the compression-entry
total.}
\end{table}

The compression-entry host timer accounts for approximately 0.065\% and
0.40\% of ASTC's mean task latency on AppWorld and QA, respectively.
AppWorld queue completion time decreases by 8.67\% and queue token
throughput increases by 3.64\%. These queue statistics include startup,
model loading, and warmup, and also depend on task scheduling. Per-GPU token
throughput is slightly lower for ASTC on both datasets. Peak GPU memory is
slightly lower on AppWorld and higher on QA; NVML sampling includes model
weights, workspaces, and caches. The measured gains therefore concern task
latency and the reported queue run, with distinct behavior in per-GPU
throughput and memory.

\section{Paired Uncertainty on OfficeBench and QA}
\label{app:paired-uncertainty}

We follow ACON's official task selection and subtask expansion:
OfficeBench contains 95 evaluation subtasks from 46 parent task groups,
and accuracy is successful subtasks divided by 95. The 8-Objective QA set
contains 99 tasks and 792 objectives. We use 10{,}000 paired bootstrap
replicates, resampling the same task records for Base and ASTC in each
replicate and recomputing quality differences and ratios of cumulative
resource totals. The OfficeBench intervals use the 95 subtasks as sampling
units. QA uses the 99 tasks as units, sampling all eight objectives of each
task together to preserve their dependence.

\begin{table}[H]
\caption{Paired 95\% bootstrap intervals for quality differences and workload
ratios.}
\label{tab:paired-uncertainty}
\centering
\small
\setlength{\tabcolsep}{2pt}
\renewcommand{\arraystretch}{1.15}
\begin{tabular*}{\linewidth}{@{\extracolsep{\fill}}llcc@{}}
\toprule
Dataset & Comparison & $\Delta$Acc./EM (pp) & $\Delta$F1 (pp) \\
\midrule
OfficeBench & Tools / Base & $+6.32\,[-3.16,16.84]$ & --- \\
OfficeBench & Context / Base & $+3.16\,[-6.32,12.63]$ & --- \\
8-Objective QA & Tools / Base & $+1.89\,[-2.78,6.57]$
& $+1.67\,[-4.07,7.31]$ \\
8-Objective QA & Context / Base & $+2.40\,[-2.40,7.32]$
& $+3.89\,[-1.79,9.51]$ \\
\bottomrule
\end{tabular*}

\vspace{7pt}
\begin{tabular*}{\linewidth}{@{\extracolsep{\fill}}llccc@{}}
\toprule
Dataset & Comparison & D-KV\% & Attn.\% & Gen.\% \\
\midrule
OfficeBench & Tools / Base & $42.09\,[24.29,68.34]$
& $40.16\,[26.39,58.79]$ & $61.24\,[45.92,80.55]$ \\
OfficeBench & Context / Base & $37.11\,[21.75,61.04]$
& $40.76\,[26.59,61.05]$ & $58.91\,[43.66,78.90]$ \\
8-Objective QA & Tools / Base & $57.96\,[32.74,98.63]$
& $74.29\,[54.52,99.31]$ & $82.02\,[64.41,98.68]$ \\
8-Objective QA & Context / Base & $66.83\,[37.88,118.23]$
& $77.75\,[57.96,104.03]$ & $87.40\,[68.14,109.92]$ \\
\bottomrule
\end{tabular*}
\vspace{2pt}
\parbox{\linewidth}{\footnotesize Entries are point estimates [95\% CI].
Quality differences are ASTC minus Base in percentage points; workload
ratios are percentages of Base (100\%).}
\end{table}

All quality-difference intervals in Table~\ref{tab:paired-uncertainty}
include zero. The OfficeBench workload intervals and the QA Tools workload
intervals lie below 100\%, whereas the QA Context workload intervals include
the Base level. Thus the positive quality point estimates have uncertainty,
while the evidence for resource reduction varies by dataset and scope.

\section{Comparison at a Similar Attention Budget}
\label{app:attention-budget-comparison}

We compare ASTC with an ACON UTCO operating point at a similar total
Attention Pairs budget on 8-Objective QA
(Table~\ref{tab:attention-budget-comparison}). This ACON UTCO configuration
is specific to the budget comparison. ASTC uses 11.54\% less total attention
and reaches the same 30.81\% EM point estimate; ACON UTCO has a 1.03-point
higher F1 score. With 10{,}000 paired bootstrap replicates over tasks,
sampling each task's eight objectives together, the ASTC-minus-ACON EM
difference is 0.00 percentage points (95\% CI $[-4.04,4.17]$), and the F1
difference is $-1.03$ points (95\% CI $[-5.92,3.77]$). Both quality-difference
intervals include zero.

\begin{table}[H]
\caption{ASTC and ACON UTCO at similar attention budgets on 8-Objective QA.}
\label{tab:attention-budget-comparison}
\centering
\small
\setlength{\tabcolsep}{3pt}
\renewcommand{\arraystretch}{1.08}
\begin{tabular*}{\linewidth}{@{\extracolsep{\fill}}lcccccc@{}}
\toprule
Method & EM\% $\uparrow$ & F1\% $\uparrow$
& \shortstack{Attention\\Pairs (T) $\downarrow$}
& Attn.\% $\downarrow$ & D-KV\% $\downarrow$ & Gen.\% $\downarrow$ \\
\midrule
Base & 28.91 & 37.51 & 2.0785 & 100.00 & 100.00 & 100.00 \\
ASTC & \textbf{30.81} & 39.18 & \textbf{1.5442}
& \textbf{74.29} & \textbf{57.96} & \textbf{82.02} \\
ACON UTCO & \textbf{30.81} & \textbf{40.21} & 1.7457 & 83.99 & 93.20 & 101.0 \\
\bottomrule
\end{tabular*}
\end{table}

\section{Matched-Length LLMLingua-2 Control}
\label{app:matched-length-control}

We construct a controlled hard-token variant using a frozen LLMLingua-2
scorer, sharing Qwen3.5-9B, Tool compression, protection rules, the Stage-3
boundary, and generation settings with ASTC. For a common input prefix, if
ASTC aggregates a span of $m$ positions into $k$ soft representatives, the
control retains the hidden states of $k$ original positions according to
LLMLingua-2 scores. All protected positions remain unchanged, so physical
retained length is exactly matched on the shared input. End-to-end
trajectories can still diverge after these controlled compression choices.
Attention Pairs includes auxiliary scorer calls.

\begin{table}[H]
\caption{ASTC and a LLMLingua-2 control with matched scope, protection, and
retained length.}
\label{tab:matched-length-control}
\centering
\small
\setlength{\tabcolsep}{2pt}
\renewcommand{\arraystretch}{1.08}
\begin{tabular*}{\linewidth}{@{\extracolsep{\fill}}llcccc@{}}
\toprule
Dataset & Method & Acc.\% $\uparrow$ & D-KV\% $\downarrow$
& Attn.\% $\downarrow$ & Gen.\% $\downarrow$ \\
\midrule
\multirow{2}{*}{AppWorld} & ASTC & \textbf{33.33} & 81.37
& \textbf{80.56} & 94.66 \\
& LLMLingua-2 (matched) & 26.19 & \textbf{77.05} & 98.78 & \textbf{90.16} \\
\addlinespace[2pt]
\multirow{2}{*}{OfficeBench} & ASTC & \textbf{54.74} & 42.09
& \textbf{40.16} & 61.24 \\
& LLMLingua-2 (matched) & 49.47 & \textbf{35.56} & 52.73 & \textbf{55.79} \\
\bottomrule
\end{tabular*}
\end{table}

Under this length control, ASTC completes 12 more AppWorld tasks, a
7.14-percentage-point success gain (paired 95\% CI $[2.98,11.90]$).
On OfficeBench, it completes five more subtasks, a 5.26-point gain
(95\% CI $[-1.05,11.58]$). The AppWorld interval excludes zero, while the
OfficeBench interval includes zero. Including auxiliary scoring, ASTC uses
18.44\% and 23.85\% fewer total Attention Pairs than the matched control on
AppWorld and OfficeBench, respectively. Table~\ref{tab:matched-length-control}
shows that the control has lower Decode KV and Generated Tokens, while ASTC
has higher accuracy point estimates and lower total attention workload.

\section{AppWorld Test-Challenge}
\label{app:challenge}

We evaluate all 417 AppWorld Test-Challenge tasks using the existing
Qwen3.5-9B, Stage-3, Tool-scope configuration and compression thresholds,
without retuning on Challenge. Table~\ref{tab:challenge} reports a success
rate of 19.90\% for ASTC and 20.86\% for Base. The paired success-rate
difference is $-0.96$ percentage points (95\% CI $[-4.32,2.40]$).
Attention Pairs decrease by 20.57\%, with a Base-relative ratio of 79.43\%
(95\% CI $[74.17\%,85.12\%]$). Attention workload therefore decreases on the
harder split, while the task-success difference remains uncertain.

\begin{table}[H]
\caption{ASTC on all 417 AppWorld Test-Challenge tasks without retuning.}
\label{tab:challenge}
\centering
\small
\setlength{\tabcolsep}{4pt}
\renewcommand{\arraystretch}{1.08}
\begin{tabular*}{\linewidth}{@{\extracolsep{\fill}}lcccc@{}}
\toprule
Method & Acc.\% $\uparrow$ & D-KV\% $\downarrow$
& Attn.\% $\downarrow$ & Gen.\% $\downarrow$ \\
\midrule
Base & \textbf{20.86} & 100.00 & 100.00 & 100.00 \\
ASTC & 19.90 & \textbf{87.13} & \textbf{79.43} & \textbf{95.34} \\
\bottomrule
\end{tabular*}
\end{table}

\section{Adapted FastKV and KVzip Baselines}
\label{app:cache-comparators}

We extend the AppWorld comparison to FastKV~\citep{jo2026fastkv} and
KVzip~\citep{kim2025kvzip}, adapted to the frozen Qwen3.5-9B model.
We evaluate all 168 Test-Normal tasks with official scoring, greedy
decoding, at most 50 agent steps, 768 output tokens per action, and a
32{,}768-token context, following the AppWorld protocol in
Appendix~\ref{app:benchmark-configurations}. Both adapted baselines are
training-free, with no weight updates, and apply their KV operations only to
the eight full-attention layers at indices $3,7,11,15,19,23,27,31$.

FastKV uses capacity 512, a Token-Selective Propagation (TSP) layer of 15,
and TSP length 2{,}048, with window size 8, max pooling, and kernel size 7.
KVzip uses retention ratio 0.3 and reconstruction chunk size 2{,}000. Its
adaptation protects the system instructions and current user task while
compressing earlier actions and tool observations.

\begin{table}[H]
\caption{Training-free cache comparators adapted to Qwen3.5-9B on AppWorld.}
\label{tab:cache-comparators}
\centering
\small
\setlength{\tabcolsep}{2pt}
\renewcommand{\arraystretch}{1.08}
\begin{tabular*}{\linewidth}{@{\extracolsep{\fill}}lcccccc@{}}
\toprule
Method & \shortstack{Training-\\free} & Successes
& Acc.\% $\uparrow$ & D-KV\% $\downarrow$
& Attn.\% $\downarrow$ & Gen.\% $\downarrow$ \\
\midrule
Base & --- & 54/168 & 32.14 & 100.00 & 100.00 & 100.00 \\
ASTC & \cmark & \textbf{56/168} & \textbf{33.33} & 81.37
& \textbf{80.56} & 94.66 \\
FastKV (adapted) & \cmark & 0/168 & 0.00 & \textbf{2.53}
& 561.23 & \textbf{55.40} \\
KVzip (adapted) & \cmark & 19/168 & 11.31 & 87.98 & 5{,}976.41 & 145.73 \\
\bottomrule
\end{tabular*}
\vspace{2pt}
\parbox{\linewidth}{\footnotesize Workload is normalized to Base (100\%).
FastKV and KVzip results correspond to the Qwen3.5-9B adaptations described
above.}
\end{table}

Table~\ref{tab:cache-comparators} reports 0 and 19 successes for the adapted
FastKV and KVzip baselines, compared with 56 for ASTC. Both baselines reduce
Decode KV, but their end-to-end Attention Pairs rise to 561.23\% and
5{,}976.41\% of Base, respectively. ASTC improves task success while reducing
all three workload measures. Under this closed-loop protocol, reducing the
decoding cache alone is insufficient to preserve task quality or reduce
total attention work; the comparison evaluates complete agent trajectories
under the stated adaptation settings.

\clearpage

\raggedbottom
\section{Experimental Setup and Protocols}
\label{app:experimental-details}

\subsection{General Setup}
\label{app:experimental-setting}

\noindent\textbf{Benchmarks.} We evaluate ASTC on three long-horizon agentic
benchmarks with distinct interaction structures. AppWorld provides a simulated
ecosystem of everyday applications and APIs for interactive coding agents; we
use its full 168-task Test-Normal split~\citep{trivedi2024appworld}.
OfficeBench covers realistic office-automation workflows that require planning,
application switching, and grounded actions across productivity
software~\citep{wang2024officebench}. 8-Objective QA groups eight Natural
Questions into each task and requires an agent to gather evidence with a
Wikipedia search tool before returning all eight
answers~\citep{kwiatkowski2019natural,zhou2025mem1,kang2026acon}.
Together, they test compression across API interaction, office automation,
and multi-question information gathering.

\noindent\textbf{Implementation.} The three main benchmarks start from
Qwen3.5-9B. Base and ASTC keep its parameters frozen and share the agent
scaffold, prompts, decoding policy, context limit, and action budget.
External compressors use the benchmark task splits and evaluators; their
model services and adaptation-specific limits are detailed in
Appendix~\ref{app:baseline-reproduction}. In particular, LCLM adds trained
Bridge and decoder LoRA parameters, so its evaluated decoder is adapted
rather than identical to the frozen Base.

\noindent\textbf{Model checkpoints.} The main experiments use
\texttt{Qwen/Qwen3.5-9B} at Hugging Face revision
\texttt{c202\allowbreak 2362\allowbreak 3576\allowbreak 2e1c\allowbreak 871a\allowbreak d0cc\allowbreak b60c\allowbreak 8ee5\allowbreak ba33\allowbreak 7b9a}~\citep{qwen2026qwen359b}.
The transfer study also uses \texttt{moonshotai/\allowbreak Kimi-\allowbreak Linear-\allowbreak 48B-\allowbreak A3B-\allowbreak Instruct}
at revision
\texttt{e1df\allowbreak 551a\allowbreak 4471\allowbreak 57d4\allowbreak 658b\allowbreak 573f\allowbreak 9a69\allowbreak 5d57\allowbreak 6585\allowbreak 90e9}~\citep{moonshot2025kimilinear}.

\noindent\textbf{Bounded pooling.} Selected positions are divided into at
most $M$ groups in causal order. In each group, an existing member closest
to the mean direction anchors an initial trim by relative $\ell_2$ error.
ASTC then averages the remaining states in FP32, casts the mean to the
runtime dtype, and accepts a group of at least two members only if its
maximum post-cast relative error is at most $\tau_h$. Rejected and trimmed
positions remain hard tokens. This check bounds hidden-state reconstruction,
not downstream attention scores.

\noindent\textbf{Baselines.} Base is the uncompressed acting LLM.  We compare
ASTC with representative
hard-token approaches, including prompt-based deletion (Prompting), LLMLingua
~\citep{jiang2023llmlingua}, LLMLingua-2~\citep{pan2024llmlingua2}, and the UT
and UTCO variants of ACON~\citep{kang2026acon}, as well as LCLM, a trained
soft-token compressor~\citep{li2026endtoend}. These external-compressor
comparisons are reported on AppWorld and OfficeBench; 8-Objective QA also
includes Prompting, both LLMLingua variants, both ACON variants, and LCLM. The newly
reported AppWorld LLMLingua and OfficeBench Prompting and LLMLingua runs
compress observations; other archived configurations use their saved history
or observation policies. ASTC evaluates both tool-observation and
full-context compression in-model. These original operating points do not
by themselves match every baseline's compression target and retained length;
the separately controlled comparison is in
Appendix~\ref{app:matched-length-control}.
The exact eligibility, protection, and cache handling of ASTC (Context)
are specified in Appendix~\ref{app:context-compression}.
Baseline preparation, inference policies, and auxiliary-count coverage are
specified in Appendix~\ref{app:baseline-reproduction}. We distinguish
model-parameter adaptation from ACON's textual guideline optimization and
from loading a pretrained LLMLingua scorer. The latter two services do not
perform additional parameter training in our reproduction.

\noindent\textbf{Workload metrics.} Decode KV counts query--cache interactions
during autoregressive decoding. Attention Pairs includes prefill and recorded
auxiliary model calls, while Generated Tokens counts discrete outputs from
acting and rewriting models. Missing auxiliary traces and reconstructed
components are identified in Appendix~\ref{app:baseline-reproduction}; an
unrecorded scorer is not treated as zero-cost.
These are analytical counts, not measured latency or GPU operations. Each
workload metric in Table~\ref{tab:main-results} is normalized to Base within
its benchmark. Measured latency, throughput, and GPU memory are reported
separately in Appendix~\ref{app:runtime-memory}.

\subsection{Benchmark-Specific Configurations}
\label{app:benchmark-configurations}

For the three main benchmarks in Table~\ref{tab:main-results}, we follow the
benchmark tasks, agent scaffolds, model-serving settings, inference limits, and
evaluation procedures of ACON~\citep{kang2026acon}.
The benchmark settings and ASTC-specific compression policies are detailed
below.

\noindent\textbf{AppWorld.} We evaluate all 168 AppWorld 0.1.3.post1
Test-Normal tasks under this ACON-aligned setup. The frozen Qwen3.5-9B agent uses
greedy decoding, at most 50 actions, 768 output tokens per action, and a
32{,}768-token logical context. The model seed is 20260914, and task $i$ uses
seed $20260912+i$ in the released task order. ASTC operates after zero-based
GDN layer 10 and before attention layer 11, retaining at most 16 soft tokens
per tool observation. Its fixed reader is calibrated from 22 records across
12 Train/Dev tasks, without weight updates. Each head uses leading eigenvectors
of trace-normalized tool-end state covariances; the ridge and RMS stabilizers
are $\lambda=10^{-4}$ and $\epsilon=10^{-6}$. The compression thresholds for
eligible data are $\chi<0.60$, $\psi<0.80$, and merge error $\leq0.50$.
Catalog observations remain verbatim and are not compressed. A commit must
save at least 256 downstream token-layer positions and 5\% of eligible tool
positions. API contracts, catalogs, and exact-value fields are protected.

\noindent\textbf{OfficeBench.} We follow ACON's official Test list and
subtask expansion, yielding 95 evaluation subtasks from 46 parent task
groups. Accuracy is the number of successful subtasks divided by 95. We use
Qwen3.5-9B, seed 20260920, and the official environment and evaluator.
The agent uses at most 50 iterations, 768 output tokens per action, and a
32{,}768-token context. OfficeBench uses the same Stage-3 boundary as
AppWorld, with thresholds $\chi<0.55$, $\psi<0.70$, and merge error
$\leq0.55$ for both data and catalog branches. In ASTC (Tools), only later
byte-identical repetitions of an application-interface scaffold are eligible
for compression; the first copy, changed scaffolds, observations, task
instructions, system text, and prior actions remain explicit. Base and ASTC
share the prompt and action budget. ASTC (Context) extends the eligible
bodies as detailed in Appendix~\ref{app:context-compression}, using the same
OfficeBench thresholds.

\noindent\textbf{8-Objective QA.} All methods evaluate the same 99 ACON
Test tasks, each with eight objectives, for 792 question--answer targets.
For Base and ASTC, the frozen Qwen3.5-9B
acting model uses BF16 with vLLM 0.18.0, a 131{,}072-token model limit,
engine seed 42, and at most 30 agent iterations. ASTC
compresses eligible tool content at the AppWorld layer boundary; the reported
policy uses
$(\chi,\psi)=(0.55,0.75)$ for ordinary tool data. EM and F1 score the eight answers per task. All methods
share the benchmark task IDs and evaluator; historical baseline serving and
adaptation details are specified in Appendix~\ref{app:baseline-reproduction}.
Workload is normalized to the corresponding Base run, with the accounting
coverage noted there.

\subsection{Baseline Reproduction and Cost Accounting}
\label{app:baseline-reproduction}

We adapt the public baseline implementations to the Qwen3.5-9B acting
checkpoint and the benchmark evaluators described above. We distinguish three
forms of preparation: using an already pretrained scoring model, optimizing a
textual compression guideline on training tasks, and updating compressor or
decoder parameters. These have different costs and should not all be described
as parameter training. The settings below describe our implementations rather
than assert identity with every setting in the original papers. The
matched-scope, matched-protection, and matched-length LLMLingua-2 comparison is
a separate control in Appendix~\ref{app:matched-length-control}.

\paragraph{ACON: guideline optimization and fixed test-time inference.}
We use ACON's UT and UTCO guideline-optimization procedures
\citep{kang2026acon}, with Qwen3.5-9B serving the acting and online compression
roles in our reproduction; the retained observation-policy optimizer calls
also identify the model as \texttt{qwen35-9b}. Training-task trajectories and the initial
guideline provide the optimizer inputs. The UT procedure proposes revised
guidelines from trajectory analyses, evaluates the candidates with agent
rollouts, and selects a guideline using the retained regression results.
UTCO starts from the selected UT guideline and additionally optimizes the
length of compressed histories or observations, with candidate generation,
evaluation, and selection. The selected guideline is then fixed for Test;
test-time compression still requires model calls to rewrite the eligible
history or observation as text. Neither the UT nor UTCO reproduction updates
the acting model or compressor weights. We do not run ACON's optional
compressor/agent LoRA distillation stages for these two baselines.

Offline preparation therefore includes trajectory collection, guideline
optimizer calls, and candidate evaluation/selection; online inference includes
the agent and its actual compression calls. We account for these phases
separately. For a concrete retained preparation record, the AppWorld
observation-policy UTCO optimizer logged 69 calls with 665{,}733 input and
50{,}487 output tokens, totaling 716{,}220 tokens. This is an optimizer-only
subtotal, not the complete cost of UTCO from initialization: it excludes the
reused UT stage, training-task rollouts, and candidate evaluations. The
corresponding initial-guideline Train90 rollout alone logged 7{,}611{,}026
input-plus-output tokens. History-policy and observation-policy search runs
have separate logs and must not be combined as if they were the same run.

\paragraph{LLMLingua v1: causal token scoring.}
We use the public \texttt{PromptCompressor} implementation
\citep{jiang2023llmlingua}, replacing its causal scoring backbone with
Qwen3.5-9B. The scoring service loads this model in BF16 and performs no
additional parameter training. It retains discrete tokens rather than
producing soft representations. The recorded service calls
\texttt{compress\_prompt} with a requested retention rate of 0.30,
\texttt{iterative\_size} equal to the input's token count (at least one),
token-level filtering enabled, and sentence- and context-level filtering
disabled. These are explicit reproduction choices; in particular, using
Qwen3.5-9B as the scorer is an adaptation of the public method. The
observation-policy integration leaves ineligible short observations unchanged;
the AppWorld and OfficeBench observation thresholds are 1{,}024 and 64
tokens, respectively. Requested retention is a local compressor setting,
not a claim that the complete agent context or end-to-end workload is 30\%
of Base.

\paragraph{LLMLingua-2: pretrained token classification.}
We use the public implementation~\citep{pan2024llmlingua2} with a frozen
pretrained classifier. Specifically, we enable \texttt{use\_llmlingua2=True}
and load \texttt{microsoft/\allowbreak llmlingua-\allowbreak 2-\allowbreak xlm-\allowbreak roberta-\allowbreak large-\allowbreak meetingbank}.
The checkpoint is an
\texttt{XLM\allowbreak Roberta\allowbreak For\allowbreak Token\allowbreak Classification} model with 24 layers, hidden size
1{,}024, and 16 attention heads. It was already trained by its authors;
we do not fine-tune it on our benchmark tasks. Our service uses retention
rate 0.30 and passes newline and question mark as \texttt{force\_tokens}.
The observation-policy runs use the same 1{,}024-token AppWorld and 64-token
OfficeBench triggers as above. Their compressed text is returned to the
agent; ASTC's recurrent-state selector, protected-position mask, and soft-token
pooling are not invoked by these ordinary baseline services. The separate
matched-length control intentionally supplies ASTC's target positions and
protection mask before applying LLMLingua-2 scores, and is labeled as such.

\paragraph{LCLM: pretrained components and Qwen adaptation.}
We initialize from the archived public
\nolinkurl{latent-context/0.6b-4b-LCLM-4x} compressor
\citep{li2026endtoend}. We retain and freeze its 0.6B encoder and pretrained
$1{,}024\!\rightarrow\!2{,}560$ adapter. To use the same initial
Qwen3.5-9B backbone as ASTC, we replace the original 4B decoder with Qwen3.5-9B
and add a Bridge consisting of RMS normalization, a
$2{,}560\!\rightarrow\!4{,}096$ linear layer, GELU, and a
$4{,}096\!\rightarrow\!4{,}096$ linear layer. Only this new Bridge and
decoder LoRA parameters are trainable; the original Qwen weights remain
frozen. LoRA has rank 8, alpha 16, dropout 0.05, and no trained bias. It
targets \texttt{q\_proj}, \texttt{k\_proj}, \texttt{v\_proj},
\texttt{o\_proj}, \texttt{in\_proj\_qkv}, \texttt{in\_proj\_z},
\texttt{out\_proj}, \texttt{gate\_proj}, \texttt{up\_proj}, and
\texttt{down\_proj}. Thus LCLM starts from the same acting checkpoint but
uses a task-adapted decoder at evaluation; it is not an unchanged frozen-actor
baseline.

For the archived AppWorld adaptation, interaction trajectories yield
2{,}526 next-action examples from 159 records spanning 29 training tasks,
with 1{,}411 validation examples from 98 records spanning 19 disjoint tasks.
The inference-aligned adaptation optimizes teacher-forced next-action loss
plus an auxiliary recap loss with weight 0.3; no future-trajectory target is
enabled. The tested rolling-history adaptation explicitly initializes its
Bridge and LoRA from this inference-aligned checkpoint and then trains on
runner-shaped, action-only targets, with older events represented as latents
and the most recent three turns supplied as text. This is the initialization
chain used by the evaluated AppWorld checkpoint. Both adaptation stages use
AdamW with learning rate $2\times10^{-4}$, weight decay 0.01, gradient-norm
clipping at 1.0, and seed 42. The rolling-history launcher runs one epoch,
capped at 2{,}526 updates, with one example per update.

The AppWorld Table~\ref{tab:main-results} row uses the archived v19 one-shot
runner and its rolling-history Bridge+\allowbreak LoRA checkpoint; the official replay
passes all tests on 11/168 tasks. At inference, the task instruction and
older history are concatenated and encoded in windows of at most 1{,}024
tokens. Consecutive groups of four
encoder states are mean-pooled, passed through the frozen adapter and new
Bridge, and supplied to Qwen as soft embeddings. The trigger is at least 12
older turns eligible for archiving, excluding the most recent three turns;
these three turns remain verbatim, compression is applied
at most once per task, and interactions added afterward remain verbatim.
The nominal $4\times$ factor therefore applies to encoded older history,
not the complete context. This archived launcher permits 50 interactions
and 128 new tokens per action, rather than ASTC's 768-token action limit;
its action parsing and progress guards are part of the adapter implementation.
The runner-shaped training encodes archived events separately, whereas this
one-shot inference runner encodes their concatenation, so the input
constructions are not identical.
These settings should be considered when interpreting the adaptation result,
which does not by itself establish a general limitation of LCLM.

OfficeBench uses a separately adapted Bridge+\allowbreak LoRA checkpoint. Its retained
adaptation data contain 2{,}533 training examples from 261 records across
18 task families and 260 validation examples from 24 records across two
disjoint families, with action/recap supervision in the inference-aligned
format. The QA8 one-pass adaptation is another
benchmark-specific run: its launcher trains Bridge+\allowbreak LoRA for one epoch capped
at 1{,}400 updates, using AdamW at $2\times10^{-4}$, weight decay 0.01,
seed 42, a 2{,}048-token memory budget, and action/recap target caps of
96/48 tokens. It does not establish that the AppWorld, OfficeBench, and QA8
results came from one shared fine-tuned checkpoint.

\paragraph{Auxiliary Attention Pairs: definition and coverage.}
An auxiliary scorer is not free simply because its parameters are frozen.
For LLMLingua v1, a causal scoring forward with $q$ new positions and a
cached prefix of length $c$ contributes
\begin{equation}
 A_{\mathrm{v1}}(q,c)
 =8\left(qc+\frac{q(q+1)}{2}\right),
\end{equation}
where eight is the number of full-attention layers in Qwen3.5-9B. Counts
are summed over actual forwards, including repeated scoring windows and
recomputation; cached prefix--prefix pairs are not counted again unless
recomputed. For LLMLingua-2, bidirectional attention in the 24-layer
classifier contributes
\begin{equation}
 A_{\mathrm{v2}}=24\sum_j n_j^2,
\end{equation}
where $n_j$ is the valid, non-padding length of each executed window,
including repeated windows and special tokens processed by the model.
These definitions count token-position pairs per full-attention layer,
not attention heads, GPU FLOPs, or recurrent GDN operations. Token selection
by a classifier does not itself generate autoregressive text tokens; model
calls that rewrite context do.

The formulas specify how auxiliary forwards should be accounted for; they
do not recover missing measurements. In particular, retained QA8 LLMLingua
service logs do not record every internal $q,c$ pair or classifier window
length. Their available agent-side Attention counts are lower bounds until
those scorer traces are included. QA8 Prompting/ACON compression costs use
the retained history-based reconstruction when native service usage is
missing and are estimates rather than exact forward traces. The retained
QA8 LCLM one-pass diagnostic counts Qwen-decoder work only, omits encoder
work, and produces no nonempty final predictions; its small counts should
not be interpreted as successful end-to-end compression. Normalizing
historical runs to the latest Base also does not establish identical
serving kernels across runs. We therefore distinguish accounting definitions,
recorded totals, and reconstructed or missing components, and report measured
latency and memory separately from analytical workload.

\subsection{ASTC with Context-Wide Eligibility}
\label{app:context-compression}

\paragraph{Scope of the Context variant.}
ASTC (Tools) makes eligible tool-observation spans available to the selector.
ASTC (Context) extends this eligibility to the bodies of completed
\texttt{system}, \texttt{user}, \texttt{assistant}, and \texttt{tool}
messages in the actual model input. It uses the same native innovation and
readability signals, fixed reader, and bounded hidden-state pooling as the
corresponding Tools configuration. It introduces no learned compressor or
additional weight updates. Context-wide eligibility does not mean that every
message or token is compressed: protected positions and candidates rejected
by either signal test or the pooling bound remain hard tokens.

Only closed message bodies are considered. Role labels, chat-template
delimiters, special tokens, and the unfinished assistant generation prefix
remain explicit. Pooling respects individual span boundaries and causal
order; it does not merge tokens across unrelated messages. The stored
transcript, original token IDs, and executable action text are not rewritten
as shorter text. Compression changes the model's intermediate sequence and
downstream cache, while the environment receives the original generated
action. A generated assistant action becomes eligible only after it has
closed, before a later generation step.

\paragraph{Protection rules.}
The scope extension retains the existing Tool masks rather than marking all
tool output compressible. In AppWorld, catalog observations remain verbatim
and are not compressed. Protected API contracts, authentication
observations, and exact-value fields retain their original provenance guards.
Completed assistant actions classified as authentication or contract access
are protected as whole blocks. For newly eligible message bodies, the
implementation protects token pieces overlapping numbers or identifiers
containing digits or underscores. AppWorld additionally preserves action
fences and reasoning/tool-result delimiters. These are explicit recognition
rules, not a claim to identify every semantically important string.

OfficeBench preserves the first canonical copy of each application-interface
scaffold and permits the later byte-identical copies to remain eligible as
in Tools. Other completed message-body spans can enter the ordinary-data
branch, with literal and protocol protection. In QA8, the mask is constructed
over the bodies actually produced by the Qwen chat template, using token
offsets and literal guards. This matters because rendering can remove older
assistant reasoning. The proxy checks the input-token digest and prompt
length against the vLLM worker; it does not locate eligible spans by searching
for source text that is absent from the rendered input.

\paragraph{Fixed selector settings.}
Table~\ref{tab:context-settings} records the settings used for both Tools and
Context within each benchmark. All use Stage 3, after zero-based GDN layer
10 and before attention layer 11, with at most 16 soft representatives per
eligible span. The original grouping, post-cast relative-error check,
minimum net saving, and protected-position rules remain active. The
minimum-saving settings are 256 downstream token-layer positions and 5\%
of the eligible span. We do not introduce a separate $\chi/\psi$ sweep for
the Context rows. Newly eligible AppWorld non-Tool content follows the data
branch; the QA8 worker likewise loads the data branch for its message spans.

\begin{table}[H]
\caption{Selector thresholds for the reported ASTC Tools and Context runs.
The merge bound is the maximum accepted post-cast relative hidden-state error.}
\label{tab:context-settings}
\centering
\small
\setlength{\tabcolsep}{7pt}
\begin{tabular}{@{}lccc@{}}
\toprule
Benchmark & $\chi$ & $\psi$ & Merge bound \\
\midrule
AppWorld & 0.60 & 0.80 & 0.50 \\
OfficeBench & 0.55 & 0.70 & 0.55 \\
8-Objective QA & 0.55 & 0.75 & 0.50 \\
\bottomrule
\end{tabular}
\end{table}

\paragraph{Cached assistant content and replay.}
Extending the mask alone is insufficient to compress an assistant action
whose tokens have already entered a cache. The reported AppWorld Context
runtime therefore retains native Stage-3 boundary representations from
decoding. At the next generation, it combines the newly closed assistant
body with the appended observation, while preserving the append-only
states of layers 0--10 and the native signals for the current Tool span.
If cached assistant positions are actually merged, layers 11--31 are
restored to their saved prefix and rebuilt from the retained boundary
representations and new suffix. Stale upper-layer KV entries are not reused
after changing those positions. If no cached positions are merged, the
existing cache is reused. Auxiliary GDN-only signal extraction and every
upper-layer rebuild are retained in the measured execution. The reported
AppWorld Context run applies the original selector bounds without an
additional veto based solely on whether the current prefill pays back that
rebuild immediately; its cost is included even when a merge is accepted.

OfficeBench uses a persistent prefill-prefix snapshot and replays the local
completed-assistant/new-input suffix when needed, rather than rebuilding
the entire accumulated transcript at every step. QA8 uses the vLLM
integration with request-specific eligibility metadata and a compact
physical view of layers 11--31. Its recorded serving configuration disables
cross-request prefix caching and enables chunked prefill. Consequently, we
do not attribute cross-request prefix-cache reuse to the QA8 Context result;
its actual repeated prefills are included in the workload.

\paragraph{Accounting and interpretation.}
The early layers still process the logical input before the merge boundary.
Decode KV and Attention Pairs use the actual physical lengths at each
full-attention layer, including replayed attention forwards. Signal
extraction, selection, pooling, scope preparation, snapshots, and replay
also contribute to measured latency; the analytical pair count does not
represent all of their GPU work. Thus broader eligibility need not reduce
total Attention Pairs, wall-clock time, or peak memory relative to Tools.
Runtime measurements are reported in Appendix~\ref{app:runtime-memory},
separately from analytical workload.

The Context rows in Table~\ref{tab:main-results} use the same benchmark task
sets as their Tools references. In the selected AppWorld Context run,
the net numbers of removed positions are 5{,}361 for Assistant and
234{,}519 for Tool; System and User positions enter the candidate set but
produce no accepted net removal under these thresholds. This distinguishes
the implemented scope from the compression actually accepted. AppWorld
Context solves 54/168 tasks, compared with Tools' 56/168; OfficeBench Context
solves 49/95, compared with 52/95; and QA8 Context has EM/F1 of
31.31/41.40, compared with 30.81/39.18. Context therefore tests the scope
extension rather than establishing uniform dominance over Tools. Paired
uncertainty should be interpreted at the task level as described in
Appendix~\ref{app:paired-uncertainty}, keeping QA8's eight objectives together.

\subsection{Additional-Dataset Protocols}
\label{app:additional-protocols}

\noindent\textbf{RULER.} We evaluate the same 32K inputs in Base and ASTC
with the frozen Qwen3.5-9B model and a 32,768-token context limit. RULER
scores answers by task-specific string matching and averages over its 13 task
configurations. We use 20 examples per task, rather than the reference
configuration's 500. Thus the reported accuracy is a 260-example aggregate;
matching aggregate scores do not imply identical results on every task or at
the larger default sample size. ASTC uses the Stage-3 data branch with
$(\chi,\psi)=(0.60,0.80)$ and merge-error threshold $0.50$. Decode KV and
Attention Pairs are analytical query--cache and causal-pair counts; Gen. is
generated tokens, all normalized to Base.

\noindent\textbf{LongBench v2.} We use all 503 four-choice questions in the
public \texttt{train} split~\citep{bai2024longbenchv2}. Base and ASTC share
the frozen Qwen3.5-9B model and a 262,144-token context limit; examples that
exceed the limit are truncated before either method processes them. Of the
503 examples, 100 require truncation. The inference configuration permits at
most 16 new tokens; the recorded runs generate two tokens per example in both
variants. ASTC uses the Stage-3 catalog branch with
$(\chi,\psi)=(0.55,0.74)$, merge-error threshold $0.55$, and minimum net
removal of one token. Overall and subgroup accuracy are exact-match rates for
the extracted choice. Reported KV counts and cache bytes are analytical
logical workload measures, not latency or measured peak device memory.

For MBPP, we use the sanitized task snapshot from
Archer~\citep{he2026archer}.

\noindent\textbf{MBPP.} We use its 427 sanitized MBPP
problems~\citep{austin2021mbpp}.
This interactive code-repair protocol is not numerically comparable with
one-shot MBPP evaluations. The model submits a complete Python
program, receives public-test feedback, and revises it; a third attempt is
allowed if the second fails. The final program is scored once. Both variants
use greedy decoding, at most 768 generated tokens per attempt, and a
32,768-token context limit. ASTC uses its Stage-3 calibration, thresholds,
and 16-soft-token bound. Initial prompts, completed code, and public-test
feedback are eligible for compression, while special control tokens remain
explicit. Both variants re-prefill the full transcript before each turn;
workload includes these repeated prefills and all decoding calls.

\end{document}